\documentclass{article}

\usepackage{arxiv}

\usepackage[utf8]{inputenc} 
\usepackage[T1]{fontenc}    
\usepackage{hyperref}       
\usepackage{url}            
\usepackage{booktabs}       
\usepackage{amsfonts}       
\usepackage{nicefrac}       
\usepackage{microtype}      
\usepackage{lipsum}		
\usepackage{graphicx}
\usepackage[dvipsnames]{xcolor}
\usepackage{float}
\usepackage{graphicx}
\usepackage{subfig}

\usepackage{amsmath}
\usepackage{amsthm}
\usepackage{amssymb}
\usepackage{stmaryrd}
\usepackage{stackrel}
\usepackage{graphicx}
\usepackage{wasysym}
\usepackage{mathtools}
\usepackage{multirow}
\usepackage{caption} 
\makeatletter

\floatstyle{ruled}
\newfloat{algorithm}{tbp}{loa}
\providecommand{\algorithmname}{Algorithm}
\floatname{algorithm}{\protect\algorithmname}

\theoremstyle{plain}
\newtheorem{thm}{\protect\theoremname}
\theoremstyle{definition}
\newtheorem{defn}{\protect\definitionname}
\theoremstyle{definition}
\newtheorem{problem}{\protect\problemname}
\theoremstyle{plain}
\newtheorem{lem}{\protect\lemmaname}
\theoremstyle{definition}

\theoremstyle{remark}
\newtheorem{rem}{\protect\remarkname}
\theoremstyle{plain}

\theoremstyle{plain}

\usepackage{cite}
\usepackage{algpseudocode}
\usepackage{setspace}
\usepackage{color}
\usepackage{comment}
\usepackage{tikz}
\tikzset{>=latex}
\usetikzlibrary{fit,automata,positioning,angles,quotes}

\makeatother

\usepackage{babel}
\providecommand{\corollaryname}{Corollary}
\providecommand{\definitionname}{Definition}
\providecommand{\examplename}{Example}
\providecommand{\lemmaname}{Lemma}
\providecommand{\problemname}{Problem}
\providecommand{\remarkname}{Remark}
\providecommand{\theoremname}{Theorem}

\usepackage{lipsum}

\newcommand\blfootnote[1]{%
	\begingroup
	\renewcommand\thefootnote{}\footnote{#1}%
	\addtocounter{footnote}{-1}%
	\endgroup
}

\SetSymbolFont{stmry}{bold}{U}{stmry}{m}{n}

\def\BibTeX{{\rm B\kern-.05em{\sc i\kern-.025em b}\kern-.08em
		T\kern-.1667em\lower.7ex\hbox{E}\kern-.125emX}}
	
\begin{document}
\title{Modular Deep Reinforcement Learning for Continuous Motion Planning
	with Temporal Logic}

\date{} 					

\author{Mingyu Cai$^{1}$, Mohammadhosein Hasanbeig$^{2}$, Shaoping Xiao$^{1}$, Alessandro Abate$^{2}$ and
	Zhen Kan$^{3}$}

\renewcommand{\headeright}{}
\renewcommand{\undertitle}{}
\renewcommand{\shorttitle}{}
\blfootnote{$^{1}$Department of Mechanical Engineering, The University of Iowa,
	Iowa City, IA, USA.}
\blfootnote{$^{2}$Department of Computer Science, University of Oxford, Parks
	Road, Oxford, UK.}
\blfootnote{$^{3}$Department of Automation, University of Science and Technology
	of China, Hefei, Anhui, China.}
\blfootnote{$^{4}$Email: mingyu-cai@uiowa.edu, hosein.hasanbeig@icloud.com, shaoping-xiao@uiowa.edu,}

\blfootnote{\hspace{1.1cm}alessandro.abate@cs.ox.ac.uk, zkan@ustc.edu.cn.}

\maketitle

\thispagestyle{fancy}

\begin{abstract}
This paper investigates the motion planning of autonomous dynamical systems
modeled by Markov decision processes (MDP) with unknown transition
probabilities over continuous state and action spaces. Linear
temporal logic (LTL) is used to specify high-level tasks over
infinite horizon, which can be converted into a limit deterministic
generalized B\"uchi automaton (LDGBA) with several accepting sets. The
novelty is to design an embedded product MDP (EP-MDP) between the LDGBA
and the MDP by incorporating a synchronous tracking-frontier function
to record unvisited accepting sets of the automaton, and to facilitate the satisfaction
of the accepting conditions. 
The proposed LDGBA-based reward shaping and discounting schemes for the model-free reinforcement learning (RL) only depend on the EP-MDP states and can overcome the issues of sparse rewards. Rigorous analysis shows that any RL method that optimizes the
expected discounted return is guaranteed to find an optimal policy whose traces maximize the satisfaction probability.
A modular deep deterministic policy gradient (DDPG) is then developed
to generate such policies over continuous state and action spaces. The performance of our framework
is evaluated via an array of OpenAI gym environments.

\global\long\def\Dist{\operatorname{Dist}}%
\global\long\def\Inf{\operatorname{Inf}}%
\global\long\def\Sense{\operatorname{Sense}}%
\global\long\def\Eval{\operatorname{Eval}}%
\global\long\def\Info{\operatorname{Info}}%
\global\long\def\ResetRabin{\operatorname{ResetRabin}}%
\global\long\def\Post{\operatorname{Post}}%
\global\long\def\Acc{\operatorname{Acc}}%
\global\long\def\True{\operatorname{True}}%
\global\long\def\False{\operatorname{False}}%
\end{abstract}


\section{Introduction}
The goal of motion planning is to generate valid configurations such
that robotic systems can complete pre-specified tasks.
In practice, however, dynamics uncertainties impose great challenges to motion planning.
Markov decision processes (MDP) are often employed to model such uncertainties
as transition probabilities. Growing
research has been devoted to studying the motion planning modelled as an MDP when these transition probabilities are initially unknown. Reinforcement learning (RL) is a sequential decision-making
process that learns optimal action policies for an unknown  MDP via gathering experience samples from the MDP \cite{Sutton2018}. RL has
achieved impressive results over the past few years, but often the learned solution is difficult to understand and examine by humans. Two main challenges existing
in many RL applications are: (i) the design of an appropriate reward shaping
mechanism to ensure correct mission specification, and (ii) the increasing sample complexity
when considering continuous state and action spaces.

Temporal logics offer rich expressivity in describing complex tasks beyond
traditional go-to-goal navigation for robotic systems \cite{Baier2008}. Specifically, motion planning under linear temporal logic (LTL) constraints attracted
growing research attention in the past few years \cite{Guo2015,Vasile2020,Cai2020c}. Under
the assumption of full knowledge of the MDP model, one common objective is to
maximize the probability of accomplishing the given LTL task \cite{Ding2014a,Lacerda2019,Kloetzer2015}.
Once this assumption is relaxed, model-based RL is employed \cite{Sadigh2014,Fu2014,Wang2015} by treating LTL specifications as reward shaping schemes to generate
policies that satisfy LTL tasks by explicitly learning unknown transition probabilities
of the MDP. This means that a model of the MDP is inferred over which an optimal policy is synthesized. However, scalability is a pressing issue for applying model-based approaches due to the
need to store the learned model. On the other hand, by relaxing the need to construct an MDP model, model-free RL is recently adopted where appropriate
reward shaping schemes are proposed \cite{lcrl,Cai2020,Hasanbeig2019a,Hasanbeig2019,Hahn2019,Bozkurt2020,cautiousRL,Oura2020,Cai2020d,Cai2021soft}.
Signal temporal logic is also proposed in \cite{Aksaray2016} and \cite{Venkataraman2020}
where a task-guided reward function
is proposed based on the robustness degrees. Despite the recent progresses, the aforementioned
works can not handle many real-world applications that perform in
high-dimensional continuous state and action spaces. In a pioneer
work \cite{Mnih2015}, Deep Q Network (DQN) addressed
high-dimensional state space and is capable of human-level performance on many Atari video games. However, DQN can only
handle discrete and low-dimensional action spaces. %
By leveraging actor-critic methods, deep networks,
and the policy gradient methods, deep deterministic policy gradient (DDPG)
was proposed to approximate optimal policies over a continuous action
space to improve the learning performance \cite{Lillicrap2015}.

In this work, we consider motion planning under LTL
task specifications in continuous state and action spaces when the MDP is fully unknown. An unsupervised
one-shot and on-the-fly DDPG-based motion planning framework is developed
to learn the state of an underlying structure without explicitly constructing
the model. The high-level LTL task is converted to a limit deterministic
generalized B\"uchi automaton (LDGBA) \cite{Sickert2016} acting as
task-guided reward shaping scheme and decomposing complex tasks into low-level
and achievable modules. In particular, we consider LTL specifications
over the infinite horizon, whose behaviors can be regarded as a repetitive
pattern \cite{Smith2011,Cai2020b}, which consists of infinite rounds
of visits of the accepting sets of LDGBA.

\textbf{Related works: }When considering deep RL with formal methods,
deterministic finite automaton (DFA) was applied as reward machines
in \cite{deepsynth}, \cite{Icarte2018} and \cite{Camacho2019}, where DQN was employed
to generate optimal policies. In \cite{Li2019}, a truncated linear
temporal logic was considered and its robustness degree was used as
the reward signal to facilitate learning, based on which proximal policy optimization
(PPO) was applied to obtain the optimal policy. However, \cite{deepsynth,Icarte2018,Camacho2019,Li2019,rens2020learning,memarian2020active} only consider tasks over
finite horizons. In contrast, this work extends previous research to
tasks over the infinite horizon, where finite horizon motion planning
can be regarded as a special case of the infinite horizon setting. Along
this line of research, the most relevant works include \cite{Yuan2019,Gao2019,Hasanbeig2019b,Hasanbeig2020,Wang2020,kazemi2020formal,lavaei2020formal}.
In \cite{Gao2019}, LTL constraints were translated to Deterministic
Rabin Automata (DRA), which might fail to find policies that maximize LTL satisfaction probability \cite{Hahn2019}. The work in \cite{Wang2020} proposed a binary
vector to record the visited accepting sets of LDGBA and designed
a varying reward function to improve the satisfaction of LTL specifications.
However, the maximum probability of task satisfaction cannot be guaranteed
and the standard DDPG algorithm cannot distinguish the sub-task module
in \cite{Wang2020}, resulting in an unsatisfactory success rate.
Modular DDPG that jointly optimizes LTL sub-policies was first introduced
in \cite{Yuan2019,Hasanbeig2020}. However, \cite{Yuan2019,Hasanbeig2020} do not explicitly record visited or unvisited accepting
sets of LDGBA in each round of repetitive pattern over the infinite horizon,
which might be essential for synthesising a pure deterministic policy \cite{Oura2020}. 

\textbf{Contributions:} The contributions of this work are multi-fold.
In contrast to most existing works that consider either a discrete state
space or a discrete action space, this paper proposes a modular DDPG architecture integrated with potential functions \cite{Ng1999}. This allows our method to generate optimal policies that efficiently solve LTL motion planning of an unknown MDP over continuous state and
action spaces. The novelty is to construct an embedded
product MDP (EP-MDP) to record unvisited accepting sets of the automaton at each round
of the repeated visiting pattern, by introducing a tracking-frontier function. 
Such a design ensures task completion by encouraging the satisfaction
of LDGBA accepting conditions. To facilitate learning of optimal
policies, the designed reward is enhanced with potential functions that effectively guide the agent toward task satisfaction without adding extra hyper-parameters to the algorithm. Unlike \cite{Oura2020}, rigorous analysis
shows that the maximum probability of task satisfaction can be guaranteed.
Compared to approaches based on limit deterministic B\"uchi automata (LDBA), e.g., \cite{Hahn2019,Bozkurt2020}, LDGBA
has several accepting sets while LDBA only has one accepting set which can result in sparse rewards during training. Consequently,
our approach can maintain a higher density of training rewards by
assigning positive rewards to the accepting states. Moreover, the proposed method does not require the construction of full EP-MDP, and its states can be obtained on-the-fly
to generate optimal policies.

In summary, our approach can find the
optimal policy in continuous state and action spaces to satisfy LTL
specifications over infinite horizon with approximate maximum
probability. 

\section{Preliminaries}

\subsection{Continuous Labeled MDP and Reinforcement Learning\label{subsec:Labeled-MDP}}

A continuous labeled MDP is a tuple $\mathcal{M}=\left(S,A,p_{S},\Pi,L,\varLambda\right)$,
where $S\subseteq\mathbb{R}^{n}$ is a continuous state space, $A\subseteq\mathbb{R}^{m}$
is a continuous action space, $\Pi$ is a set of atomic propositions,
$L:S\shortrightarrow2^{\Pi}$ is a labeling function, and $p_{S}:\mathfrak{B}\left(\mathbb{R}^{n}\right)\times A\times S\shortrightarrow\left[0,1\right]$
is a Borel-measurable conditional transition kernel such that $p_{S}\left(\left.\cdot\right|s,a\right)$
is a probability measure of $s\in S$ and $a\in A$ over the Borel
space $\left(\mathbb{R}^{n},\mathfrak{B}\left(\mathbb{R}^{n}\right)\right)$,
where $\mathfrak{B}\left(\mathbb{R}^{n}\right)$ is the set of all
Borel sets on $\mathbb{R}^{n}$. The transition probability $p_{S}$
captures the motion uncertainties of the agent. It is assumed that
$p_{S}$ is not known \textit{a priori}, and the agent can only observe
its current state and the associated label of the current state. 

A deterministic policy $\boldsymbol{\xi}$ of continuous MDP is a function $\boldsymbol{\xi}:S\rightarrow {A}$
that maps each state to an action over the action space $A$.
The MDP $\mathcal{M}$ evolves by taking an action $a_{i}$ based on the rule $\xi_{i}$ at each
stage $i$%
, and thus the control policy $\boldsymbol{\xi}=\xi_{0}\xi_{1}\ldots$
is a sequence of rules, which yields a path $\boldsymbol{s}=s_{0}s_{1}s_{2}\ldots$
over $\mathcal{M}$ with the transition $s_{i}\overset{a_{i}}{\rightarrow}s_{i+1}$
that exists, i.e., $s_{i+1}$ belongs to 
 the smallest Borel set $B$ such
that $p_{S}\left(\left.B\right|s_{i},a_{i}\right)=1$. 
The control policy $\boldsymbol{\xi}$ is memoryless if each $\xi_{i}$
only depends on its current state, and $\boldsymbol{\xi}$ is a finite
memory policy if $\xi_{i}$ depends on its past history.

Let $\varLambda:S\times A\times S\shortrightarrow\mathbb{R}$ denote
a reward function. Given a discounting function $\gamma:S\times A\times S\shortrightarrow\mathbb{R}$,
the expected discounted return under policy \textbf{$\boldsymbol{\xi}$} starting
from $s\in S$ is defined as $$U^{\boldsymbol{\xi}}\left(s\right)=\mathbb{E}^{\boldsymbol{\xi}}\left[\stackrel[i=0]{\infty}{\sum}\gamma^{i}\left(s_{i},a_{i},s_{i+1}\right)\cdot\varLambda\left(s_{i},a_{i},s_{i+1}\right)\left|s_{0}=s\right.\right].$$
An optimal policy $\boldsymbol{\xi}^{*}$ maximizes the expected
return for each state $s\in S$, i.e., $$\boldsymbol{\xi}^{*}=\underset{\boldsymbol{\xi}}{\arg\max}U^{\boldsymbol{\xi}}\left(s\right).$$
The function $U^{\boldsymbol{\xi}}\left(s\right)$ is often referred
to as the value function under policy \textbf{$\boldsymbol{\xi}$}.
If the MDP is not fully known, but the state and action spaces are countably finite, tabular approaches are usually
employed \cite{Watkins1992}. However, traditional tabular RL methods
are not applicable to MDPs with continuous state and action spaces.
In this work, we propose a policy gradient method that relies on deep neural
networks to parameterize the policy model. %

\subsection{LTL and Limit-Deterministic Generalized B\"uchi Automaton}

Linear temporal logic (LTL) is a formal language that is widely used to describe
complex mission tasks \cite{Baier2008}. %
The semantics of an LTL formula are interpreted over a word, which is an
infinite sequence $o=o_{0}o_{1}\ldots$ where $o_{i}\in2^{\Pi}$ for
all $i\geq0$, and $2^{\Pi}$ represents the power set of $\Pi$.
Denote by $o\models\phi$ if the word $o$ satisfies the LTL formula
$\phi$. %
{}Given an LTL specification, its satisfaction can be evaluated by
an LDGBA \cite{Sickert2016}. An LDGBA is a sub-class of generalized B\"uchi automata (GBA) that can express the set of words of an LTL formula. In this work, we restrict our attention to LTL formulas that exclude the \emph{next} temporal operator, which is not meaningful for continuous time execution~\cite{kloetzer2008fully}.
\begin{defn}
	\label{def:GBA} A GBA is a tuple $\mathcal{A}=\left(Q,\Sigma,\delta,q_{0},F\right)$,
	where $Q$ is a finite set of states, $\Sigma=2^{\Pi}$ is a finite
	alphabet, $\delta\colon Q\times\Sigma\shortrightarrow2^{Q}$ is the
	transition function, $q_{0}\in Q$ is the initial state, and $F=\left\{ F_{1},F_{2},\ldots,F_{f}\right\} $
	is the set of accepting sets where $F_{i}\subseteq Q$, $\forall i\in\left\{ 1,\ldots f\right\} $. 
\end{defn}
Denote by $\boldsymbol{q}=q_{0}q_{1}\ldots$ a run of a GBA, where
$q_{i}\in Q$, $i=0,1,\ldots$. The run $\boldsymbol{q}$ is accepted
by the GBA, if it satisfies the generalized B\"uchi acceptance condition,
i.e., $\inf\left(\boldsymbol{q}\right)\cap F_{i}\neq\emptyset$, $\forall i\in\left\{ 1,\ldots f\right\} $,
where $\inf\left(\boldsymbol{q}\right)$ denotes the infinite part
of $\boldsymbol{q}$. 
\begin{defn}
	\label{def:LDGBA} A GBA is an LDGBA if the transition function $\delta$
	is extended to $Q\times\left(\Sigma\cup\left\{ \epsilon\right\} \right)\shortrightarrow2^{Q}$,
	and the state set $Q$ is partitioned into a deterministic set $Q_{D}$
	and a non-deterministic set $Q_{N}$, i.e., $Q_{D}\cup Q_{N}=Q$ and
	$Q_{D}\cap Q_{N}=\emptyset$, where 
	\begin{itemize}
		\item the state transitions in $Q_{D}$ are total and restricted within
		it, i.e., $\bigl|\delta\left(q,\alpha\right)\bigr|=1$ and $\delta\left(q,\alpha\right)\subseteq Q_{D}$
		for every state $q\in Q_{D}$ and $\alpha\in\Sigma$, 
		\item An $\epsilon$-transition is not allowed in the deterministic set,
		i.e., for any $q\in Q_{D}$, $\delta\left(q,\epsilon\right)=\emptyset$,
		and 
		\item the accepting states are only in the deterministic set, i.e., $F_{i}\subseteq Q_{D}$
		for every $F_{i}\in F$. 
	\end{itemize}
\end{defn}
In Definition \ref{def:LDGBA}, $\epsilon$-transitions are only for state transitions from $Q_{N}$ to $Q_{D}$, without reading an alphabet. We used OWL \cite{Kretinsky2018}, an easily-accessible tool, to generate LDGBAs in this work. \textcolor{black}{}%
\textcolor{blue}{{} }In the following analysis, we use $\mathcal{A}_{\phi}$
to denote the LDBA corresponding to an LTL formula $\phi$. %

\begin{defn}
\label{def:Non-accepting Sink Component} A non-accepting sink component
$Q_{sink}\subseteq Q$ of an LDGBA is a strongly connected directed
graph induced by a set of states, such that the accepting condition
can not be satisfied if starting from any state in $Q_{sink}$.
\end{defn}

\section{Problem Statement}


\begin{figure}[!t]\centering
	\subfloat[]{{\includegraphics[width=0.40\columnwidth]{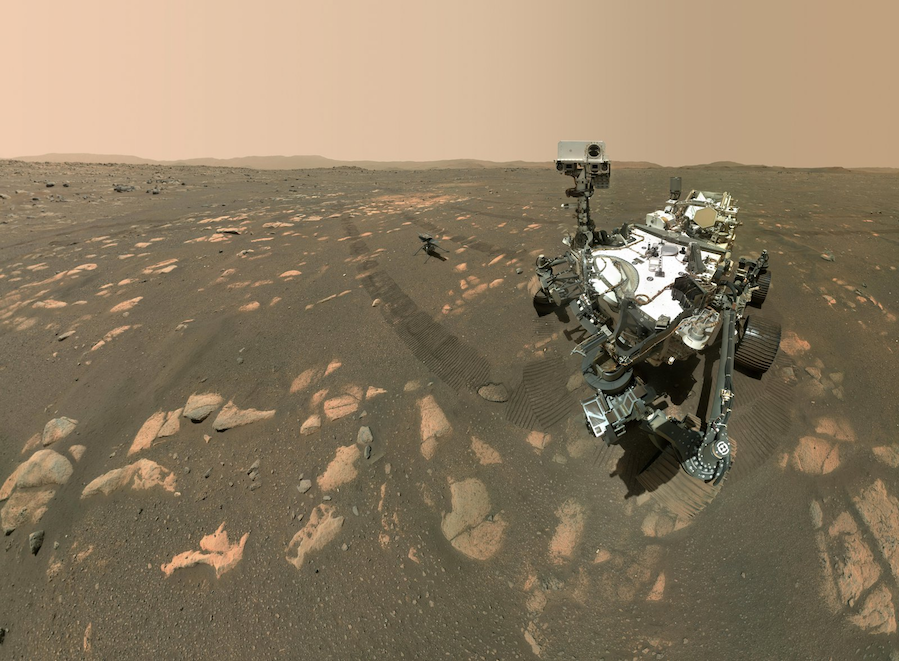} }}%
	\subfloat[]{{\includegraphics[width=0.44\columnwidth]{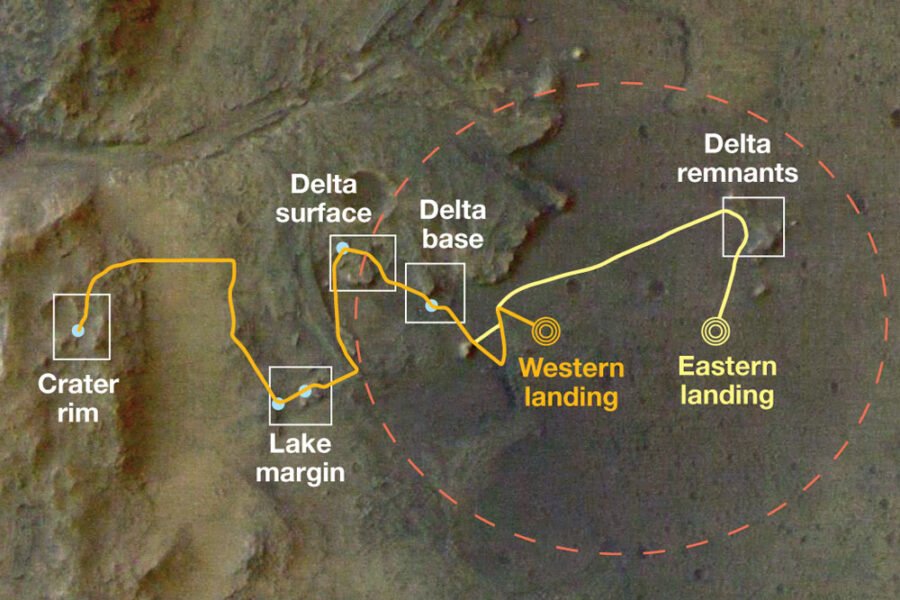} }}%
	\caption{Example of Mars exploration (courtesy of NASA)}
	\label{fig:example}
\end{figure}


Consider a robot that performs a mission described by an LTL formula
$\phi$. For instance, as shown in Fig. \ref{fig:example}, the helicopter provides the map consisting of areas of interest, based on which the ground Mars rover is tasked to visit all regions marked with rectangles. Due to the complex terrain of Mars surface, there exist motion uncertainties during the rover movement. The interaction of the robot with the environment
is modeled by a continuous MDP $\mathcal{M}$, which can be used to model decision-making problems for general robotic systems. Under a policy $\boldsymbol{\xi}=\xi_{0}\xi_{1}\ldots$,
the induced path over $\mathcal{M}$ is $\boldsymbol{s}_{\infty}^{\boldsymbol{\xi}}=s_{0}\ldots s_{i}s_{i+1}\ldots$.
Let $L\left(\boldsymbol{s}_{\infty}^{\boldsymbol{\xi}}\right)=l_{0}l_{1}\ldots$
be the sequence of labels associated with $\boldsymbol{s}_{\infty}^{\boldsymbol{\xi}}$
such that $l_{i}\in L\left(s_{i}\right)$. Denote the satisfaction relation by $L\left(\boldsymbol{s}_{\infty}^{\boldsymbol{\xi}}\right)\models\phi$
if the induced 
trace satisfies
$\phi$. The probabilistic satisfaction under the policy $\xi$ from
an initial state $s_{0}$ can then be defined as
\begin{equation}
{\Pr{}_{M}^{\boldsymbol{\xi}}\left(\phi\right)=\Pr{}_{M}^{\boldsymbol{\xi}}\left(L\left(\boldsymbol{s}_{\infty}^{\boldsymbol{\xi}}\right)\models\phi\big|\boldsymbol{s}_{\infty}^{\boldsymbol{\xi}}\in\boldsymbol{S}_{\infty}^{\boldsymbol{\boldsymbol{\xi}}}\right),}\label{eq:probabilistic-satisfaction}
\end{equation}
where $\boldsymbol{S}_{\infty}^{\boldsymbol{\xi}}$ is a set of admissible
paths from the initial state under the policy ${\boldsymbol{\xi}}$. 

\textbf{Assumption1.} It is assumed that there exists at least one policy whose induced traces satisfy the task $\phi$ with non-zero probability.

Assumption 1 is mild and widely employed in the literature (cf. \cite{Sadigh2014,Hahn2019,Bozkurt2020}), which indicates $\phi$ can be completed. Based on Assumption 1, this work considers the following problem. 

\begin{problem}
	\label{Prob1}Given an LTL-specified task $\phi$ and a continuous-state continuous-action
	labeled MDP $\mathcal{M}$ with unknown transition probabilities, the goal is to learn a policy $\boldsymbol{\xi}^{*}$
	that maximizes the satisfaction probability in the limit, i.e., $\boldsymbol{\xi}^{*}=\underset{\boldsymbol{\xi}}{\arg\max}\Pr{}_{M}^{\boldsymbol{\xi}}\left(\phi\right)$. 
\end{problem}

\section{Automaton Design\label{sec:Automaton-Analysis}}

To address Problem \ref{Prob1}, Section \ref{subsec:PMDP} presents
the construction of the EP-MDP
between an MDP and an LDGBA. The advantages of incorporating EP-MDP
are discussed in Section \ref{subsec:Property_relax}.

\subsection{Embedded Product MDP\label{subsec:PMDP}}

Given an LDGBA $\mathcal{A}=\left(Q,\Sigma,\delta,q_{0},F\right)$,
a tracking-frontier set $T$ is designed to keep track of unvisited
accepting sets. Particularly, $T$ is initialized as $T_{0}=F$ and $\mathcal{B}$ is a Boolean variable, which is
then updated according to the following rule:
\begin{equation}
{\left(T,\mathcal{B}\right)}=f_{V}\left(q,T\right)=\left\{ \begin{array}{cc}
\left(T\setminus F_{j},\operatorname{False}\right), & \text{if }q\in F_{j}\text{ and }F_{j}\in T,\\
\left(F\setminus F_{j},\operatorname{True}\right) & \text{if }\ensuremath{q\in F_{j}\text{ and }T=\emptyset},\\
\left(T,\operatorname{False}\right), & \text{otherwise. }
\end{array}\right.\label{eq:Trk-fontier}
\end{equation}

Once a state $q \in F_{j}$ is visited, $F_{j}$ will be removed
from $T$ by rendering function $f_{V}\left(q,T\right)$.
If $T$ becomes empty before visiting set $F_{j}$, it will
be reset as $F\setminus F_{j}$. Since the acceptance condition of
LDGBA requires to infinitely visit all accepting sets, we call it
one round if all accepting sets have been visited (i.e., a round ends
if $T$ becomes empty). The second output of $f_{V}\left(q,T\right)$
is to indicate whether all accepting sets have been visited in current
round, which is applied in Section \ref{subsec:RL} to design the
potential function. Based on (\ref{eq:Trk-fontier}),
the EP-MDP is constructed as follows.

\begin{defn}
\label{def:P-MDP} Given an MDP $\mathcal{M}$ and an LDGBA $\mathcal{A}_{\phi}$,
the EP-MDP is defined as $\mathcal{P}=\mathcal{M}\times\mathcal{A}_{\phi}={\color{red}{\color{black}\left(X,U^{\mathcal{P}},p^{\mathcal{P}},x_{0},F^{\mathcal{P}},T,f_{V},\mathcal{B}\right)}}$,
where $X=S\times Q\times2^{F}$ is the set of product states and $2^{F}$ denotes all subsets of $F$ for $\mathcal{A}_{\phi}$, i.e.,
$x=\left(s,q,T\right)\in X$ ; $U^{\mathcal{P}}=A\cup\left\{ \epsilon\right\} $
is the set of actions, where the $\epsilon$-actions are only allowed
for transitions from $Q_{N}$ to $Q_{D}$; $x_{0}=\left(s_{0},q_{0},T_{0}\right)$
is the initial state; $F^{\mathcal{P}}=\left\{ F_{1}^{\mathcal{P}},F_{2}^{\mathcal{P}}\ldots F_{f}^{\mathcal{P}}\right\} $
where $F_{j}^{\mathcal{P}}=\left\{ \left(s,q, T\right)\in X\bigl|q\in F_{j}\land F_{j} \subseteq T\right\} $,
$j=1,\ldots f$, is a set of accepting states; $p^{\mathcal{P}}$
is the transition kernel for any transition $p^{\mathcal{P}}\left(x,u^{\mathcal{P}},x'\right)$
with $x=\left(s,q,T\right)$ and $x'=\left(s',q',T\right)$ such that
: (1) $p^{\mathcal{P}}\left(x,u^{\mathcal{P}},x'\right)=p_{S}\left(\left.s'\right|s,a\right)$
if $s^{\prime}\backsim p_{S}\left(\left.\cdot\right|s,a\right)$,
${\delta}\left(q,L\left(s\right)\right)=q^{\prime}$ where
$u^{\mathcal{P}}=a\in A$ \; (2) $p^{\mathcal{P}}\left(x,u^{\mathcal{P}},x'\right)=1$
if $\ensuremath{u^{\mathcal{P}}\in\left\{ \epsilon\right\} }$, $q'\in{\delta}\left(q,\epsilon\right)$
and $s'=s$; and (3) $p^{\mathcal{P}}\left(x,u^{\mathcal{P}},x'\right)=0$
otherwise. After completing each transition $q'={\delta}\left(q,\alpha\right)$ based on $p^{\mathcal{P}}$,
$T$ is synchronously updated as ${\left(T,\mathcal{B}\right)=f_{V}\left(q',T\right)}$
by (\ref{eq:Trk-fontier}).
\end{defn}

The state-space is embedded with the tracking-frontier set $T$ that can be practically represented via one-hot encoding based on the indices of the accepting set. Compared to the standard construction of product MDPs, any state of the accepting set in EP-MDP, e.g., $\left(s,q, T\right)\in F_{j}^{\mathcal{P}}$, requires the automaton state to satisfy $q\in F_{j}\land F_{j} \subseteq T$, and embedded tracking frontier set $T$ is updated based on (\ref{eq:Trk-fontier}) after each transition. Consequently, to satisfy the accepting condition, the agent is encouraged to visit all accepting sets.

In this work, the EP-MDP is only used for theoretical analyses and it is not
constructed in practice. The EP-MDP captures the intersections
between all feasible paths over $\mathcal{M}$ and all words accepted
to $\mathcal{{A}}_{\phi}$, facilitating the identification
of admissible agent actions that satisfy task $\phi$. 

\begin{algorithm}
\caption{\label{Alg:EP-MDP} generating a random run of
EP-MDP}

\scriptsize

\singlespacing

\begin{algorithmic}[1]

\Procedure {Input: } {$\mathcal{M}$, $\mathcal{\mathcal{A}_{\phi}}$,
$f_{V},T$ and length $L$}

{Output: } { A valid run $\boldsymbol{x}_{\mathcal{P}}$ with length
$L$ in $\mathcal{P}$}

\State set $x_{0}=\left(s_{0},q_{0},T_{0}\right)$, $s_{cur}=s_{0}$
and $\boldsymbol{x}_{\mathcal{P}}=\left(x_{0}\right)$

\State set $T=F$ and $count=1$

\While { $count\leq L$ }

\State set $x_{suc}=\emptyset$ 

\State obtain $s_{suc}$ from $p_{S}\left(\left.\cdot\right|s_{cur},a\right)$
by random sampling action $a$

\For { each successor $q_{next}$ of $q_{cur}$ in $\mathcal{\mathcal{A}_{\phi}}$ }

\If { ${\delta}\left(q_{cur},L\left(s_{cur}\right)\right)=q_{next}$
or $q_{next}\in{\delta}\left(q_{cur},\epsilon\right)$}

\State $q_{suc}\shortleftarrow q_{next}$

\State $x_{suc}\shortleftarrow\left(s_{suc},q_{suc},T\right)$

\State check if $x_{suc}$ is an accepting state

\State$T=f_{V}\left(q_{suc},T\right)$

\State \textbf{break}

\EndIf

\EndFor

\If {no successor $q_{suc}$ found } 

\State fail to generate run with length $L$

\EndIf

\State add state $x_{suc}$ to $\boldsymbol{x}_{\mathcal{P}}$ 

\State $q_{cur}\shortleftarrow q_{suc}$ and $s_{cur}\shortleftarrow s_{suc}$ 

\State $count++$;

\EndWhile

\EndProcedure

\end{algorithmic}
\end{algorithm}

Algorithm \ref{Alg:EP-MDP} shows the procedure of obtaining a valid
run $\boldsymbol{x}_{\mathcal{P}}$ on-the-fly within EP-MDP
by randomly selecting an action. After each transition, the accepting state is determined based on the Definition \ref{def:P-MDP} and the $T$ is synchronously updated (line 7-15). Such property
is the innovation of EP-MDP that encourages all accepting sets to be
visited in each round. Since the action is selected randomly, there is a non-zero probability that $\boldsymbol{x}_{\mathcal{P}}$ violates the LTL task (line 14-16).

Let $\boldsymbol{\pi}$ denote a policy over $\mathcal{P}$ and denote
by $\boldsymbol{x}_{\infty}^{\boldsymbol{\pi}}=x_{0}\ldots x_{i}x_{i+1}\ldots$
the infinite path generated by $\boldsymbol{\pi}$. A path $\boldsymbol{x}_{\infty}^{\boldsymbol{\pi}}$
is accepted if $\inf\left(\boldsymbol{x}_{\infty}^{\boldsymbol{\pi}}\right)\cap F_{i}^{\mathcal{P}}\neq\emptyset$
, $\forall i\in\left\{ 1,\ldots f\right\} $. We denote $\Pr^{\mathbf{\boldsymbol{\pi}}}\left[x\models\Acc_{p}\right]$
as the probability of satisfying the accepting condition of $\mathcal{P}$
under policy $\boldsymbol{\pi}$, and denote $\Pr_{max}\left[x\models\Acc_{p}\right]=\underset{\boldsymbol{\pi}}{\max}\Pr_{M}^{\boldsymbol{\pi}}\left(\Acc_{p}\right)$ as the maximum probability of satisfying the accepting condition of $\mathcal{P}$. Let $\boldsymbol{\boldsymbol{\pi}}^{*}$ denote an optimal policy that maximizes the expected discounted return over $\mathcal{P}$, i.e., $\boldsymbol{\pi}^{*}=\underset{\boldsymbol{\pi}}{\arg\max}U^{\boldsymbol{\pi}}\left(s\right)$. Note that the memory-less policy $\boldsymbol{\boldsymbol{\pi}}^{*}$ over $\mathcal{P}$ yields a finite-memory policy $\boldsymbol{\xi}^{*}$ over $\mathcal{M}$, allowing us to reformulate Problem \ref{Prob1} into the following:
 
\begin{problem}
\label{Prob:2} Given a user-specified LTL task $\phi$ and a general and unknown continuous-state continuous-action
labeled MDP, the goal is to
asymptotically find a policy $\boldsymbol{\boldsymbol{\pi}}^{*}$ satisfying the
acceptance condition of $\mathcal{\mathcal{P}}$ with maximum probability,
i.e., $\Pr^{\boldsymbol{\pi}^{*}}\left[x\models\Acc_{p}\right]=\Pr_{max}\left[x\models\Acc_{p}\right]$. 
\end{problem}

\subsection{Properties of EP-MDP\label{subsec:Property_relax}}

Consider a sub-EP-MDP $\mathcal{P}'_{\left(X',U'\right)}$, where
$X'\subseteq X$ and $U'\subseteq U^{\mathcal{P}}$. If $\mathcal{P}'_{\left(X',U'\right)}$
is a maximum end component (MEC) of $\mathcal{P}$ and $X'\cap F_{i}^{\mathcal{P}}\neq\emptyset$,
$\forall i\in\left\{ 1,\ldots f\right\} $, then $\mathcal{P}'_{\left(X',U'\right)}$
is called an accepting maximum end component (AMEC) of $\mathcal{P}$.
Once a path enters an AMEC, the subsequent path will stay within it
by taking restricted actions from $U'$. There exist policies such
that any state $x\in X'$ can be visited infinitely often. As a result,
satisfying task $\phi$ is equivalent to reaching an AMEC. Moreover,
a MEC that does not intersect with any accepting sets is called a rejecting
accepting component (RMEC) and a MEC intersecting with only partial accepting sets
is called a neutral maximum end component (NMEC) \cite{Baier2008}.

\begin{defn}\label{def:induced_markov_chain}
Let $MC_{\mathcal{\mathcal{P}}}^{\boldsymbol{\pi}}$ denote the Markov
chain induced by a policy $\boldsymbol{\pi}$ on $\mathcal{\mathcal{\mathcal{P}}}$,
whose states can be represented by a disjoint union of a transient
class $\ensuremath{\mathcal{T}_{\boldsymbol{\pi}}}$ and $n_R$ closed
irreducible recurrent classes $\ensuremath{\mathcal{R}_{\boldsymbol{\pi}}^{j}}$,
$j\in\left\{ 1,\ldots,n_{R}\right\} $ \cite{Durrett1999}.
\end{defn}
\begin{lem}
\label{lemma:accepting set}Given an EP-MDP $\mathcal{\mathcal{\mathcal{P}}=M}\times\mathcal{{A}}_{\phi}$
, the recurrent class $R_{\boldsymbol{\pi}}^{j}$ of $MC_{\mathcal{\mathcal{P}}}^{\boldsymbol{\pi}}$,
$\forall j\in\left\{ 1,\ldots,n_R\right\} $, induced by $\pi$ satisfies
one of the following conditions: $\ensuremath{R_{\boldsymbol{\pi}}^{j}}\cap F_{i}^{\mathcal{\mathcal{P}}}\neq\emptyset,\forall i\in\left\{ 1,\ldots f\right\} $,
or $R_{\boldsymbol{\pi}}^{j}\cap F_{i}^{\mathcal{\mathcal{P}}}=\emptyset,\forall i\in\left\{ 1,\ldots f\right\} $. 
\end{lem}
\begin{proof}
The strategy of the following proof is based on contradiction. Assume
there exists a policy such that $\ensuremath{\ensuremath{R_{\boldsymbol{\pi}}^{j}}}\cap F_{k}^{\mathcal{\mathcal{P}}}\neq\emptyset$,
$\forall k\in K$, where $K$ is a subset of $2^{\left\{ 1,\ldots f\right\} }\setminus\left\{ \left\{ 1,\ldots f\right\} ,\emptyset\right\} $.
As discussed in \cite{Durrett1999}, for each state in recurrent class,
it holds that $\stackrel[n=0]{\infty}{\sum}p^{n}\left(x,x\right)=\infty$,
where $x\in\ensuremath{\ensuremath{R_{\pi}^{j}}}\cap F_{k}^{\mathcal{\mathcal{P}}}$
and $p^{n}\left(x,x\right)$ denotes the probability of returning
from a transient state $x$ to itself in $n$ steps. This means that
each state in the recurrent class occurs infinitely often. However,
based on the embedded tracking-frontier function of EP-MDP, once $x_{k}$ is visited, the corresponding
$F_{k}$ of $F_{k}^{\mathcal{\mathcal{P}}}$ is removed from $T$,
and the tracking set $T$ will not be reset until all accepting sets
have been visited. As a result, neither $q_{k}\in F_{k}$ nor ${x_{k}=\left(s,q_{k},T\right)}\in\ensuremath{\ensuremath{R_{\boldsymbol{\pi}}^{j}}}\cap F_{k}^{\mathcal{\mathcal{P}}}$
with $s\in S$ will occur infinitely, which contradicts the property
$\stackrel[n=0]{\infty}{\sum}p^{n}\left(x_{k},x_{k}\right)=\infty$. 
\end{proof}
Lemma \ref{lemma:accepting set} indicates that, for any policy, all
accepting sets will be placed either in the transient class or in
the recurrent classes.

\section{Learning-based Control Synthesis\label{sec:Solution}}

In the following, we discuss a base reward design and present rigorous analysis to show how such a design can guide the RL-agent over the EP-MDP to find an optimal policy whose traces satisfies the LTL task with maximum probability. In order to improve the reward density, Section \ref{subsec:reward_shaping} proposes a reward shaping process via integrating a potential function under which the optimal policies remain invariant. Finally, Section \ref{subsec:RL} shows how to apply the shaped reward function with DDPG to construct a modular DDPG architecture and effectively solve Problem \ref{Prob:2}.

\subsection{Base Reward\label{subsec:RL-reward}}

Let $F_{U}^{\mathcal{\mathcal{P}}}$ denote the union of accepting states, i.e.,
$F_{U}^{\mathcal{\mathcal{P}}}=\left\{ x\in X \bigl| x\in F_{i}^{\mathcal{\mathcal{P}}},\forall i\in\left\{ 1,\ldots f\right\}\right\} $. For each transition $\left(x,u^{\mathcal{P}},x'\right)$ in the EP-MDP, the reward and discounting function only depend on current state $x$, i.e., $R\left(x,u^{\mathcal{P}},x'\right)=R\left(x\right)$ and $\gamma\left(x,u^{\mathcal{P}},x'\right)=\gamma\left(x\right)$.

Inspired by \cite{Bozkurt2020}, we propose a reward function as:
\begin{equation}
R\left(x\right)=\left\{ \begin{array}{cc}
1-r_{F}, & \text{if }x\in F_{U}^{\mathcal{\mathcal{P}}},\\
0, & \text{otherwise,}
\end{array}\right.\label{eq:reward_function}
\end{equation}

and a discounting function as 
\begin{equation}
\gamma\left(x\right)=\left\{ \begin{array}{cc}
r_{F}, & \text{if }x\in F_{U}^{\mathcal{\mathcal{P}}},\\
\gamma_{F}, & \text{otherwise,}
\end{array}\right.\label{eq:discount_function}
\end{equation}

where $r_{F}\left(\gamma_{F}\right)$ is a function of
$\gamma_{F}$ satisfying $\underset{\gamma_{F}\shortrightarrow1^{-}}{\lim}r_{F}\left(\gamma_{F}\right)=1$
and $\underset{\gamma_{F}\shortrightarrow1^{-}}{\lim}\frac{1-\gamma_{F}}{1-r_{F}\left(\gamma_{F}\right)}=0$.

Given a path $\boldsymbol{x}_{t}=x_{t}x_{t+1}\ldots$ starting from
$x_{t}$, the return is denoted by
\begin{equation}
\mathcal{D}\left(\boldsymbol{x}_{t}\right)\coloneqq\stackrel[i=0]{\infty}{\sum}\left(\stackrel[j=0]{i-1}{\prod}\gamma\left(\boldsymbol{x}_{t}\left[t+j\right]\right)\cdotp R\left(\boldsymbol{x}_{t}\left[t+i\right]\right)\right)\label{eq:DisctRetrn}
\end{equation}
where $\stackrel[j=0]{-1}{\prod}\coloneqq1$ and $\boldsymbol{x}_{t}\left[t+i\right]$
denotes the $\left(i+1\right)$th state in $\boldsymbol{x}_{t}$.
Based on (\ref{eq:DisctRetrn}), the expected return of any state
$x\in X$ under policy $\pi$ can be defined as 
\begin{equation}
U^{\boldsymbol{\pi}}\left(x\right)=\mathbb{E}^{\boldsymbol{\pi}}\left[\mathcal{D}\left(\boldsymbol{x}_{t}\right)\left|\boldsymbol{x}_{t}\left[t\right]=x\right|\right].\label{eq:ExpRetrn}
\end{equation}

A bottom strongly connected component (BSCC) of the Markov chain $MC_{\mathcal{\mathcal{P}}}^{\boldsymbol{\pi}}$ (Definition \ref{def:induced_markov_chain}) is a strongly connected component with no outgoing transitions.

\begin{lem}
\label{lemma:1} For any path $\boldsymbol{x}_{t}$ and $\mathcal{D}\left(\boldsymbol{x}_{t}\right)$
in (\ref{eq:DisctRetrn}), it holds that $0\leq\gamma_{F}\cdot\mathcal{D}\left(\boldsymbol{x}_{t}\left[t+1:\right]\right)\leq\mathcal{D}\left(\boldsymbol{x}_{t}\right)\leq1-r_{F}+r_{F}\cdot\mathcal{D}\left(\boldsymbol{x}_{t}\left[t+1:\right]\right)\leq1$,
where $\boldsymbol{x}_{t}\left[t+1:\right]$ denotes the suffix of
$\boldsymbol{x}_{t}$ starting from $x_{t+1}$. Let $BSCC\left(MC_{\mathcal{\mathcal{P}}}^{\boldsymbol{\pi}}\right)$
denote the set of all BSCCs of an induced Markov chain $MC_{\mathcal{\mathcal{P}}}^{\boldsymbol{\pi}}$ and let $X_{\mathcal{\mathcal{\mathcal{P}}}}^{\boldsymbol{\pi}}$ denotes the set of accepting states that belongs to a BSCC of $MC_{\mathcal{\mathcal{P}}}^{\boldsymbol{\pi}}$ s.t. $X_{\mathcal{\mathcal{\mathcal{P}}}}^{\boldsymbol{\pi}}\coloneqq\left\{ x\in X \bigl| x\in F_{U}^{\mathcal{\mathcal{P}}} \cap BSCC\left(MC_{\mathcal{\mathcal{P}}}^{\boldsymbol{\pi}}\right) \right\} $. Then, for any states $x\in X_{\mathcal{\mathcal{\mathcal{P}}}}^{\boldsymbol{\pi}}$,
it holds that$\underset{\gamma_{F}\shortrightarrow1^{-}}{\lim}U^{\boldsymbol{\pi}}\left(x\right)=1$.
\end{lem}
The proof of Lemma \ref{lemma:1} is omitted since it is a straightforward
extension of Lemma 2 and Lemma 3 in \cite{Bozkurt2020}, by replacing
LDBA with LDGBA. Since we apply the LDGBA with several accepting sets which might result in more complicated situations, e.g., AMEC, NMEC and RMEC, we can not obtain the same results as in \cite{Bozkurt2020}. We then establish the following theorem which
is one of the main contributions.
\begin{thm}
\label{lemma:probability} Given the EP-MDP $\mathcal{\mathcal{\mathcal{P}}=M}\times\mathcal{{A}}_{\phi}$,
for any state $x\in X$, the expected return under any policy $\pi$
satisfies 
\begin{equation}
\exists i\in\left\{ 1,\ldots f\right\} ,\underset{\gamma_{F}\shortrightarrow1^{-}}{\lim}U^{\boldsymbol{\pi}}\left(x\right)=\Pr{}^{\boldsymbol{\pi}}\left[\diamondsuit F_{i}^{\mathcal{\mathcal{P}}}\right],\label{eq:reachability}
\end{equation}

where $\Pr^{\boldsymbol{\pi}}\left[\diamondsuit F_{i}^{\mathcal{\mathcal{P}}}\right]$
is the probability that the paths starting from state $x$ will eventually
intersect a $F_{i}^{\mathcal{\mathcal{P}}}\in F^{\mathcal{P}}$. 
\end{thm}

\begin{proof}
Proof can be found in Appendix~\ref{append:1}.
\end{proof}

Next, we will show in the following sections how Lemma \ref{lemma:accepting set} and Theorem \ref{lemma:probability}
can be leveraged to enforce the RL-agent satisfying the accepting condition of $\mathcal{P}$.
\begin{thm}
\label{thm2}Consider an MDP $\mathcal{M}$ and an LDGBA $\mathcal{{A}}_{\phi}$
corresponding to an LTL formula $\phi$. Based on Assumption 1, there exists a discount factor $\underline{\gamma}$,
with which any optimization method for (\ref{eq:ExpRetrn}) with $\gamma_{F}>\underline{\gamma}$
and $r_{F}>\underline{\gamma}$ can obtain a policy $\bar{\boldsymbol{\pi}}$,
such that the induced run $r_{\mathcal{\mathcal{\mathcal{\mathcal{\mathcal{\mathcal{P}}}}}}}^{\bar{\boldsymbol{\pi}}}$
satisfies the accepting condition $\mathcal{\mathcal{\mathcal{\mathcal{P}}}}$ with non-zero probability in the limit. 
\end{thm}

\begin{proof}
Proof can be found in Appendix~\ref{Append:2}.
\end{proof}

Theorem \ref{thm2} proves that by selecting $\gamma_{F}>\underline{\gamma}$
and $r_{F}>\underline{\gamma}$, optimizing the expected return in
(\ref{eq:ExpRetrn}) can find a policy satisfying the given task $\phi$ with non-zero probability. 
\begin{thm}
\label{thm:probability} Given an MDP $\mathcal{M}$ and an LDGBA
$\mathcal{{A}}_{\phi}$, by selecting $\gamma_{F}\shortrightarrow1^{-}$,
the optimal policy in the limit $\boldsymbol{\pi}^{*}$ that maximizes the expected
return (\ref{eq:ExpRetrn}) of the corresponding EP-MDP also maximizes
the probability of satisfying $\phi$, i.e., $\Pr^{\boldsymbol{\pi}^{*}}\left[x\models Acc_{\mathcal{P}}\right]=\Pr_{max}\left[x\models Acc_{\mathcal{P}}\right]$. 
\end{thm}
\begin{proof}
Since $\gamma_{F}\shortrightarrow1^{-}$, we have $\gamma_{F}>\underline{\gamma}$
and $r_{F}>\underline{\gamma}$ from Theorem \ref{thm2}. There exists
an induced run $r_{\mathcal{\mathcal{\mathcal{\mathcal{\mathcal{\mathcal{P}}}}}}}^{\boldsymbol{\pi}^{*}}$
satisfying the accepting condition of $\mathcal{\mathcal{\mathcal{\mathcal{P}}}}$.
According to Theorem \ref{lemma:probability}, $\underset{\gamma_{F}\shortrightarrow1^{-}}{\lim}U^{\boldsymbol{\pi}^{*}}\left(x\right)$
is exactly equal to the probability of visiting the accepting sets
of an AMEC. Optimizing $\underset{\gamma_{F}\shortrightarrow1^{-}}{\lim}U^{\boldsymbol{\pi}^{*}}\left(x\right)$
is equal to optimizing the probability of entering AMECs. 
\end{proof}

\subsection{Reward Shaping \label{subsec:reward_shaping}}

\begin{figure}[!t]\centering
	{{
	    \scalebox{.8}{
			\begin{tikzpicture}[shorten >=1pt,node distance=2.5cm,on grid,auto] 
			\node[state,initial] (q_0)   {$q_0$}; 
			\node[state] (q_1) [right=of q_0, label=below:$\textcolor{blue}{\Phi_1}$]  {$q_1$};
			\node[state] (q_2) [right=of q_1, label=below:$\textcolor{blue}{\Phi_2}$]  {$q_2$};
			\node[state,accepting, label=right:$\textcolor{ForestGreen}{F_1}$] (q_3) [right=of q_2]  {$q_3$};
			\node[state] (q_4) [below=of q_0]  {$q_{sink}$};
			\path[->] 
			(q_0) edge [bend left=0] node {$T_1$} (q_1)
			(q_0) edge [loop above] node {$\lnot\left(T_1\vee U\right)$} (q_0)
			(q_0) edge [bend left=0] node {$U$} (q_4)
			(q_1) edge [bend left=0] node {$T_2$} (q_2)
			(q_1) edge [loop above] node {$\lnot\left(T_2\vee U\right)$} (q_1)
			(q_1) edge [bend left=0] node {$U$} (q_4)
			(q_2) edge [bend left=0] node {$T_3$} (q_3)
			(q_2) edge [loop above] node {$\lnot\left(T_3\vee U\right)$} (q_2)
			(q_2) edge [bend left=0] node {$U$} (q_4)
			(q_3) edge [loop above] node {$T_3$} (q_3)
			(q_4) edge [loop right] node {$\operatorname{True}$} (q_4);
			\end{tikzpicture}
			}
			$F=\{\{q_3\}\}$
			}}
		\caption{\label{fig:reward_shaping} LDGBA $\mathcal{A}_{\varphi_{P}}$ expressing $\varphi_{P}=\lozenge\left(\mathtt{T1}\land\lozenge\mathtt{\left(\mathtt{T2}\land\lozenge\mathtt{\mathtt{T}3}\right)}\right)\land\lnot\oblong\mathtt{U}$}
	\end{figure}

Since the base reward function
in Section \ref{subsec:RL-reward} is always zero for 
$x\notin F_{U}^{\mathcal{\mathcal{P}}}$, the reward signal might become sparse. To resolve this, we propose
a potential function $\Phi:X\shortrightarrow\mathbb{R}$, and transform the reward as follows:

\begin{equation}
R'\left(x,u^{\mathcal{P}},x'\right)=R\left(x\right)+\gamma\left(x\right)\cdot\Phi\left(x'\right)-\Phi\left(x\right)\label{eq:Shaped_Reward}
\end{equation}

As shown in \cite{Ng1999}, given any MDP model, e.g.,
EP-MDP $\mathcal{P}$, transforming the reward function using a potential function $\Phi\left(x\right)$ as in $R'\left(x,u^{\mathcal{P}},x'\right)$ will not change the
set of optimal policies. 
Thus, a real-valued function $\Phi\left(x\right)$ will improve the learning performance while guaranteeing that the resulted policies via $R'\left(x,u^{\mathcal{P}},x'\right)$
are still optimal with respect to the base reward function $R\left(x\right)$.

Given $\mathcal{\mathcal{\mathcal{P}}=M}\times\mathcal{A}_{\phi}={\left(X,U^{\mathcal{P}},p^{\mathcal{P}},x_{0},F^{\mathcal{P}},T,f_{V},\mathcal{B}\right)}$
with $\mathcal{\mathcal{\mathcal{A}_{\phi}}}=\left(Q,\Sigma,\delta,q_{0},F\right)$,
let $F_{U}=\left\{ q\in Q \bigl| q\in F_{i},\forall i\in\left\{ 1,\ldots f\right\} \right\} $ denote the union of automaton accepting states. For the states of $\mathcal{P}$ whose automaton states belong to $Q\setminus\left(F_{U}\cup q_{0}\cup Q_{sink}\right)$, it is desirable to assign positive rewards when the agent first visits them and assign large value of reward to the accepting states to enhance the convergence of neural network.
This is because starting from the automaton initial state, any automaton state that can reach any of the accepting sets has to be explored.
To this end, an automaton tracking-frontier set $T_{\Phi}$ is designed to keep
track of unvisited automaton components $Q\setminus\left(q_{0}\cup Q_{sink}\right)$,
and $T_{\Phi}$ is initialized as $T_{\Phi0}=Q\setminus\left(q_{0}\cup Q_{sink}\right)$. The set $T_{\Phi0}$ is then updated after each transition $\left(\left(s,q,T\right),u^{\mathcal{P}},\left(s',q',T\right)\right)$
of $\mathcal{\mathcal{P}}$ as:

\begin{equation}
f_{\Phi}\left(q',T_{\Phi}\right)=\left\{ \begin{array}{cc}
T_{\Phi}\setminus q', & \text{if }q\in T_{\Phi},\\
T_{\Phi0}\setminus q' & \text{if }\ensuremath{\mathcal{B}}=\operatorname{True},\\
T_{\Phi}, & \text{otherwise. }
\end{array}\right.\label{eq:automaton-fontier}
\end{equation}
The set $T_{\Phi}$ will only be reset when $\mathcal{B}$ in $f_{V}$ becomes $\operatorname{True}$,
indicating that all accepting sets in the current round have been visited.
Based on (\ref{eq:automaton-fontier}), the potential function $\Phi\left(x\right)$
for $x=\left(s,q,T\right)$ is constructed as:

\begin{equation}
\Phi\left(x\right)=\left\{ \begin{array}{cc}
\eta_{\Phi}\cdot(1-r_{F}), & \text{if }q\in T_{\Phi},\\
0, & \text{otherwise}
\end{array}\right.\label{eq:potential_function}
\end{equation}

where $\eta_{\Phi}> 0$ is the shaing parameter. Intuitively, the value of potential function for
unvisited and visited states in $T_{\Phi0}$ is equal to $\eta_{\Phi}\cdot(1-r_{F})$
and $0$ respectively, in each round, which will guide the system
finally reach any of the accepting sets. 

To illustrate the idea, consider an LTL formula $\varphi_{P}=\lozenge\left(\mathtt{T1}\land\lozenge\mathtt{\left(\mathtt{T2}\land\lozenge\mathtt{\mathtt{T}3}\right)}\right)\land\lnot\oblong\mathtt{U}$, e.g., the agent should always avoid obstacles and visit $\mathtt{T1}$,
$\mathtt{T2}$, $\mathtt{T3}$ sequentially. The corresponding LDGBA $\mathcal{A}_{\varphi_{P}}$
is shown in Fig. \ref{fig:reward_shaping}. Let's denote any product
state with the same automaton component and an arbitrary MDP state as
$x\left(\left\llbracket s\right\rrbracket ,q,T\right)$, where the MDP component can be different. Note that only state $q_{3}$ belongs to the accepting
set and the original reward function $R\left(x\right)=0$,
$\forall x=\left(\left\llbracket s\right\rrbracket ,q_{12},T\right)$
with. $q_{12}\neq q_{3}$. However, for $x_{1}=\left(\left\llbracket s\right\rrbracket ,q_{1},T\right)$
and $x_{2}=\left(\left\llbracket s\right\rrbracket ,q_{2},T\right)$,
we have $\Phi\left(x_{1}\right)\neq0$ and $\Phi\left(x_{2}\right)\neq0$
if the corresponding automaton component has not been visited yet (i.e., $q_{1}$, $q_{2}$ still in $T_{\Phi}$) by ($\ref{eq:automaton-fontier}$) and ($\ref{eq:potential_function}$).
For instance, given a run $\boldsymbol{x}=\left(\left\llbracket s\right\rrbracket ,q_{0},T\right)u_{0}^{\mathcal{P}}\left(\left\llbracket s\right\rrbracket ,q_{1},T\right)u_{1}^{\mathcal{P}}\left(\left\llbracket s\right\rrbracket ,q_{2},T\right)u_{2}^{\mathcal{P}}\left(\left\llbracket s\right\rrbracket ,q_{3},T\right),$
the associated shaped reward for each transition is $R'\left(\left(\left\llbracket s\right\rrbracket ,q_{0},T\right),u_{0}^{\mathcal{P}},\left(\left\llbracket s\right\rrbracket ,q_{1},T\right)\right)=R'\left(\left(\left\llbracket s\right\rrbracket ,q_{1},T\right),u_{1}^{\mathcal{P}},\left(\left\llbracket s\right\rrbracket ,q_{2},T\right)\right)=\gamma_{F}\cdot\left(1-r_{F}\right)$, where $\gamma_{F}$ is obtained based on (\ref{eq:discount_function}).
Note that every time, if the $q_{i}$ has been visited, it'll be removed
from $T_{\Phi}$, resulting in $\Phi\left(x\right)=0$ in future transitions
until $T_{\Phi}$ is reset.

\subsection{Modular Deep Deterministic Policy Gradient \label{subsec:RL}}

\begin{figure}
	\centering{}\includegraphics[scale=0.40]{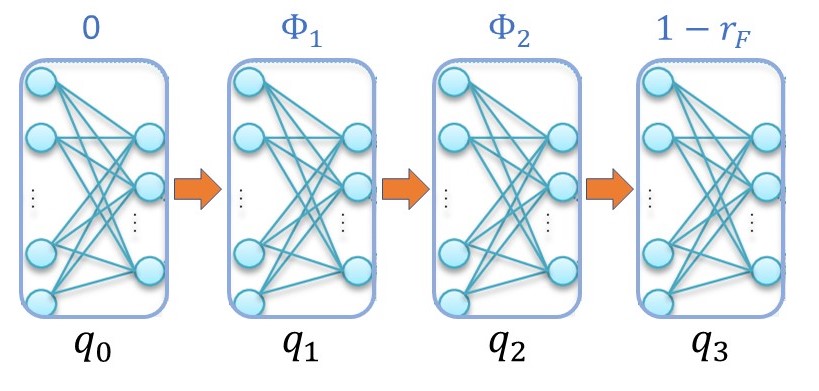}
	\caption{\label{fig:modular_DDPG} Modular DDPG for LDGBA $\mathcal{A}_{\varphi_{P}}$} in Fig. \ref{fig:reward_shaping} 
\end{figure}

\begin{algorithm}
\caption{\label{Alg2} Modular DDPG}

\scriptsize

\singlespacing

\begin{algorithmic}[1]

\Procedure {Input: } {MDP $\mathcal{M}$ , LDGBA $\mathcal{{A}}_{\phi}$}

{Output: } {modular DDPG for optimal policy $\boldsymbol{\pi}^{*}$
} 

{Initialization: } {$\left|Q\right|$ actor $\boldsymbol{\pi}_{q_{i}}\left(x\left|\theta^{u_{q_{i}}}\right.\right)$
and critic networks $Q_{q_{i}}\left(x,u^{\mathcal{P}}\left|\theta^{Q_{q_{i}}}\right.\right)$
with arbitrary weights $\theta^{u_{q_{i}}}$ and $\theta^{\boldsymbol{Q}_{q_{i}}}$
for all $q_{i}\in Q$}; {$\left|Q\right|$ corresponding target
networks $\boldsymbol{\pi}_{q_{i}}'\left(x\left|\theta^{u_{q_{i}}'}\right.\right)$
and $Q_{q_{i}}'\left(x,u^{\mathcal{P}}\left|\theta^{Q_{q_{i}}'}\right.\right)$
with weights $\theta^{u_{q_{i}}'}$ and $\theta^{Q_{q_{i}}'}$ for
each $q_{i}\in Q$, respectively}; {$\left|Q\right|$ replay buffers
$B_{q_{i}}$}; {$\left|Q\right|$ random processes noise $N_{q_{i}}$}

\State set $r_{F}=0.99$ and $\gamma_{F}=0.9999$ to determine $R\left(x\right)$
and $\gamma\left(x\right)$

\State set maximum episodes $E$ and iteration number $\tau$

\For { each episode in $E$ }

\State set $t=0$

\State sample an initial state $s_{0}$ of $\mathcal{M}$ and $q_{0}$
of $\mathcal{{A}}_{\phi}$ as $s_{t},q_{t}$

\State set $t=0$ and construct an initial product state $x_{t}=\left(s_{t},q_{t},T\right)$, \newline
		\hspace*{3.0em}where $T=T_{0}$

\While {$t\leq\tau$ }

\State select action $u_{t}^{\mathcal{P}}=\boldsymbol{\pi}_{q_{t}}\left(x\left|\theta^{u_{q_{i}}}\right.\right)+R_{q_{t}}$
based on exploitation versus \newline
		\hspace*{4.6em}exploration noise

\State execute $u_{t}^{\mathcal{P}}$, ${\delta}$ and observe
$x_{t+1}=\left(s_{t+1},q_{t+1},T\right)$, $R\left(x_{t+1}\right)$,
\newline
		\hspace*{4.6em}$\Phi\left(x_{t+1}\right)$, $\gamma\left(x_{t+1}\right)$ 

\State obtain $\Phi\left(x_{t}\right)$ based on current $T_{\Phi}$ and calculate $R'\left(x_{t},u_{t}^{\mathcal{P}},x_{t+1}\right)$

\State execute the updates via $f_{V}\left(q_{t+1},T\right)$
and $f_{\Phi}\left(q_{t+1},T_{\Phi}\right)$

\State store the sample $\left(x_{t},u_{t}^{\mathcal{P}},R'\left(x_{t},u_{t}^{\mathcal{P}},x_{t+1}\right),\gamma\left(x_{t+1}\right),x_{t+1}\right)$
\newline
		\hspace*{4.8em}in replay buffers $B_{q_{t}}$

\State mini-batch sampling $N$ data from the replay buffers $B_{q_{t}}$

\State calculate target values for each $i\in N$ as:

\[
{y_{i}=R'\left(x_{i},u_{i}^{\mathcal{P}},x_{i+1}\right)+\gamma\left(x_{i}\right)\cdot Q_{q_{i+1}}'\left(x_{i+1},u_{i+1}^{\mathcal{P}}\left|\theta^{Q_{q_{i+1}}'}\right.\right)}
\]

\State update weights $\theta^{Q_{q_{t}}}$ of critic neural network
$Q_{q_{t}}\left(x,u^{\mathcal{P}}\left|\theta^{Q_{q_{t}}}\right.\right)$ \newline
		\hspace*{4.8em}by minimizing the loss function:

\[
L=\frac{1}{N}\stackrel[i=1]{N}{\sum}\left(y_{i}-Q_{q_{t}}\left(x_{i},u_{i}^{\mathcal{P}}\left|\theta^{Q_{q_{t}}}\right.\right)\right)^{2}
\]

\State update weights $\theta^{u_{q_{t}}}$ of actor neural network
$\boldsymbol{\pi}_{q_{t}}\left(x\left|\theta^{u_{q_{t}}}\right.\right)$ \newline
		\hspace*{4.8em}by maximizing the policy gradient:

\[
\begin{aligned}\nabla_{\theta^{u_{q_{t}}}}U^{q_{t}}\thickapprox & \frac{1}{N}\stackrel[i=1]{N}{\sum}\left(\nabla_{u^{\mathcal{P}}}Q_{q_{t}}\left(x_{i},u^{\mathcal{P}}\left|\theta^{Q_{q_{t}}}\right.\right)\left|_{u^{\mathcal{P}}=\boldsymbol{\pi}_{q_{t}}\left(x_{i}\left|\theta^{u_{q_{t}}}\right.\right)}\right.\right.\\
 & \left.\cdot\nabla_{\theta^{u_{q_{t}}}}\boldsymbol{\pi}_{q_{t}}\left(x_{i}\left|\theta^{u_{q_{t}}}\right.\right)\right)
\end{aligned}
\]

\State soft update of target networks:

\begin{equation}
\begin{array}{c}
\theta^{u_{q_{t}}'}\leftarrow\tau\theta^{u_{q_{t}}}+\left(1-\tau\right)\theta^{u_{q_{t}}'}\\
\theta^{Q_{q_{t}}'}\leftarrow\tau\theta^{Q_{q_{t}}}+\left(1-\tau\right)\theta^{Q_{q_{t}}'}
\end{array}\label{eq:soft_update}
\end{equation}

\State $x_{t}\leftarrow x_{t+1}$ and $t++$
\EndWhile

\EndFor

\EndProcedure

\end{algorithmic}
\end{algorithm}

To deal with continuous-state and continuous-action MDPs, a deep deterministic
policy gradient (DDPG) \cite{Lillicrap2015} is adopted in this work to approximate the
current deterministic policy via a parameterized function $\boldsymbol{\pi}\left(x\left|\theta^{u}\right.\right)$ called actor. The actor is a deep neural network whose set of weights are $\theta^{u}$. The critic function  also applies a deep neural network with parameters $\theta^{Q}$ to approximate action-value function $Q\left(x,u^{\mathcal{P}}\left|\theta^{Q}\right.\right)$, which is updated by minimizing the following loss function:

\begin{equation}
\begin{array}{c}
L\left(\theta^{Q}\right)=\mathbb{E}_{s\sim\rho^{\beta}}^{\boldsymbol{\pi}}\left[\left(y-Q\left(x,\boldsymbol{\pi}\left(x\left|\theta^{u}\right.\right)\left|\theta^{\boldsymbol{Q}}\right.\right)\right)^{2}\right],
\end{array}\label{eq:loss_function}
\end{equation}

where $\rho^{\beta}$ is the state distribution under any arbitrary
policy $\beta$, and $y=R'\left(x,u^{\mathcal{P}},x'\right)+\gamma\left(x\right)Q\left(\left.x',u^{\mathcal{P}'}\right|\theta^{Q}\right)$
with $u^{\mathcal{P}'}=\boldsymbol{\pi}\left(x'\left|\theta^{u}\right.\right)$.
The actor can be updated by applying the chain rule to the expected
return with respect to actor parameters $\theta^{u}$ as the following
policy gradient theorem \cite{Lillicrap2015}:

\begin{equation}
\begin{array}{c}
\nabla_{\theta^{u}}U^{u}\left(x\right)\thickapprox\mathbb{E}_{s\sim\rho^{\beta}}^{\boldsymbol{\pi}}\left[\nabla_{\theta^{u}}Q\left(x,\boldsymbol{\pi}\left(x\left|\theta^{u}\right.\right)\left|\theta^{Q}\right.\right)\right]\\
=\mathbb{E}_{s\sim\rho^{\beta}}^{\boldsymbol{\pi}}\left[\nabla_{u^{\mathcal{P}}}Q\left(x,u^{\mathcal{P}}\left|\theta^{Q}\right.\right)\left|_{u^{\mathcal{P}}=\boldsymbol{\pi}\left(x\left|\theta^{u}\right.\right)}\nabla_{\theta^{u}}\boldsymbol{\pi}\left(x\left|\theta^{u}\right.\right)\right.\right].
\end{array}\label{eq:actor_update}
\end{equation}

Inspired by \cite{Yuan2019} and \cite{Hasanbeig2020}, the complex
LTL task $\phi$ can be divided into simple composable modules. Each
state of the automaton in the LDGBA is module
and each transition between these automaton states is a ``task divider''. In
particular, given $\phi$ and its LDGBA $\mathcal{{A}}_{\phi}$,
we propose a modular architecture of $\left|Q\right|$ DDPG respectively, i.e., $\boldsymbol{\pi}_{q_{i}}\left(x\left|\theta^{u}\right.\right)$
and $Q_{q_{i}}\left(x,u^{\mathcal{P}}\left|\theta^{Q}\right.\right)$
with $q_{i}\in Q$, along with their own replay buffer. Experience
samples are stored in each replay modular buffer $B_{q_{i}}$ in the
form of $\left(x,u^{\mathcal{P}},R\left(x\right),\gamma\left(x\right),x'\right)$.
By dividing the LTL task into sub-stages, the set of neural nets acts
in a global modular DDPG architecture, which allows the
agent to jump from one module to another by switching between the
set of neural nets based on transition relations of $\mathcal{{A}}_{\phi}$.

Fig. \ref{fig:modular_DDPG} shows the modular DDPG architecture corresponding to the LDGBA $\mathcal{A}_{\varphi_{P}}$ shown in Fig. \ref{fig:reward_shaping} without the sink node, where each network represents the standard DDPG structure along with an automaton state, and the transitions between each DDPG are consistent with the edges in $\mathcal{A}_{\varphi_{P}}$. The shaped reward function consisting of $R\left(x\right)$ and $\Phi\left(x\right)$ is capable of guiding the transitions among the modular neural networks.

The proposed method to solve a continuous MDP with
LTL specifications is summarized in  Alg. \ref{Alg2}, and the product states of EP-MDP are synchronized on-the-fly (line 9-12).
We assign each DDPG an individual
replay buffer $B_{q_{i}}$ and a random process noise $N_{q_{i}}$.
The corresponding weights of modular networks, i.e., $Q_{q_{i}}\left(x,u^{\mathcal{P}}\left|\theta^{\boldsymbol{Q}_{q_{i}}}\right.\right)$
and $\boldsymbol{\pi}_{q_{i}}\left(x\left|\theta^{u_{q_{i}}}\right.\right)$,
are also updated at each iteration (line 15-20). All neural networks
are trained using their own replay buffer, which is a finite-sized cache that
stores transitions sampled from exploring the environment. Since the direct
implementation of updating target networks can be unstable and divergent
\cite{Mnih2015}, the soft update (\ref{eq:soft_update}) is employed,
where target networks are slowly updated via relaxation (line 20). Note that for each iteration we first observe the output of the shaped reward function $R'$, then execute the update process via $f_{V}$
and $f_{\Phi}$ (line 10-12). Note that since DDPG leverages a deep neural network to approximate the action-value function, and that in practice we have to stop the training after finite number of steps the synthesised policy might be sub-optimal with respect to the true $\boldsymbol{\pi}^{*}$ as in Theorem \ref{thm:probability}. This is due to the nature of DDPG algorithm and non-linear function approximation in deep architectures whose analysis is out of the scope of this paper. 

\begin{rem}
	Standard embeddings, such as the one-hot and integer encoding, have
	been applied with a single DDPG network to estimate the Q-function
	\cite{Wang2020}. However,
	it is observed that single DDPG network exhibits undesirable success
	rate. The modular DDPG with the EP-MDP in this work is capable of recording
	the unvisited accepting sets for each round, which outperforms
	the result of \cite{Wang2020}.
\end{rem}

\section{Case Studies\label{sec:Case}}

The developed modular DDPG-based control synthesis is implemented
in Python, and all simulation videos and source codes can be found
in our Github repository\footnote{\url{https://github.com/mingyucai/Modular_Deep_RL_E-LDGBA}}. Owl \cite{Kretinsky2018} is used to convert LTL specifications
to LDGBA. Simulations are carried out on a machine with 2.80 GHz quad-core
CPU, 16 GB RAM, and an external Nvidia RTX 1080 GPU. We test the algorithm
in $4$ different environments in OpenAI gym and with $6$ LTL tasks. In particular, we first assign complex tasks, formulated via LTL, to Ball-Pass and CartPole problems. To show the advantage of EP-MDP with modular
DDPG, we compare the method with (i) the standard product
MDP using modular DDPG and (ii) the EP-MDP using
standard DDPG. Then, we test the scalability of our algorithm in pixel-based Mars exploration environments. Finally, we analyze and compare the performances of the modular and standard DDPG algorithms via success rates of task accomplishments.

\subsection{Ball-Pass and Cart-Pole\label{sec:Case1}}

\begin{figure}
	\centering{}\includegraphics[scale=0.35]{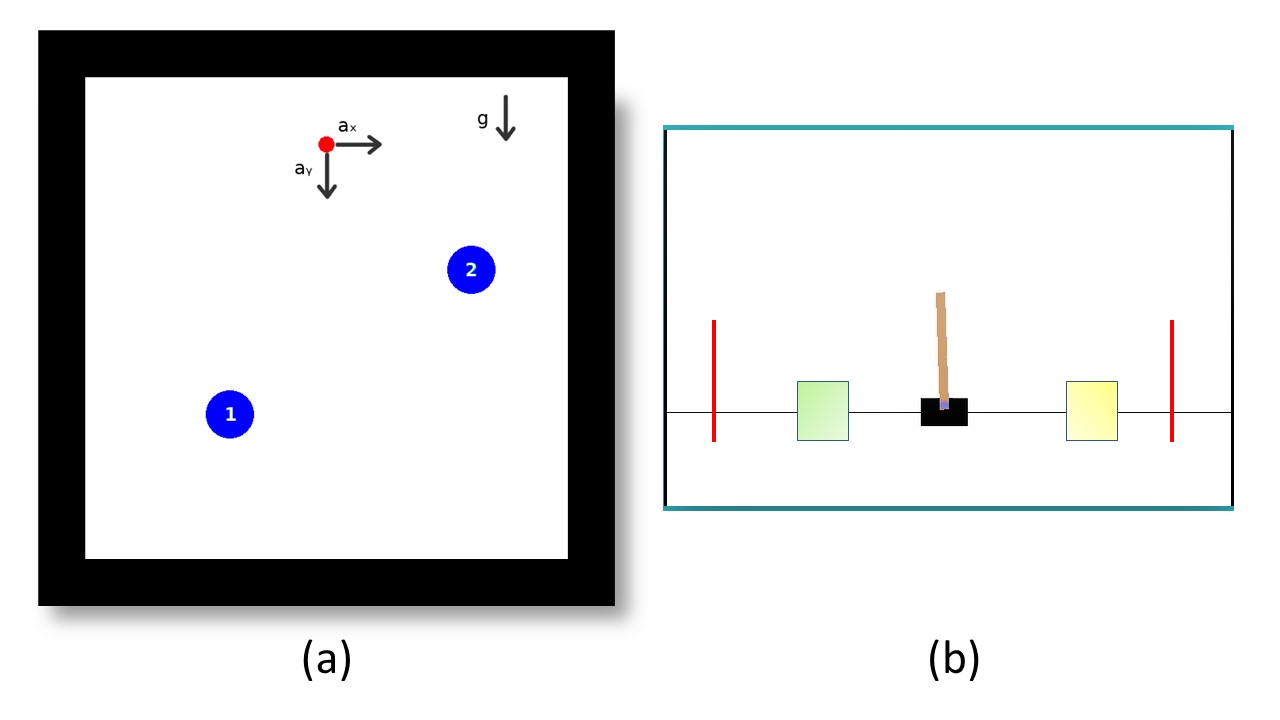}\caption{\label{fig:case_study3} (a) Ball-Pass and (b) Cart-Pole OpenAI
		environment.}
\end{figure}

\begin{figure}[!t]\centering
	\subfloat[]{{
	    \scalebox{.8}{
			\begin{tikzpicture}[shorten >=1pt,node distance=2.5cm,on grid,auto] 
			\node[state,initial] (q_0)   {$q_0$}; 
			\node[state,accepting] (q_1) [above right=of q_0, label=above:$\textcolor{ForestGreen}{F_1}$]  {$q_1$};
			\node[state,accepting, label=below:$\textcolor{ForestGreen}{F_2}$] (q_2) [below right=of q_0]  {$q_2$};
			\path[->] 
			(q_0) edge [bend left=15] node {$r_1$} (q_1)
			(q_1) edge [bend left=15] node {$r_f$} (q_0)
			(q_0) edge [loop below] node {$r_f$} (q_0)
			(q_1) edge [loop right] node {$r_1$} (q_1)
			(q_2) edge [loop right] node {$r_2$} (q_2)
			(q_0) edge [bend left=15] node {$r_2$} (q_2)
			(q_2) edge [bend left=15] node {$r_f$} (q_0)
			(q_1) edge [bend left=15] node {$r_2$} (q_2)
			(q_2) edge node {$r_1$} (q_1);
			\end{tikzpicture}
			}
			$F=\{\{q_1\}, \{q_2\}\}$
			}}
	\subfloat[]{{
		\scalebox{.8}{
			\begin{tikzpicture}[shorten >=1pt,node distance=2.5cm,on grid,auto] 
			\node[] (s_0)   {}; 
			\node[state, initial] (s_0) [above right=of s_0, label=above:$\textcolor{blue}{\{r_f\}}$]  {$s_0$};
			\node[state] (s_1) [above right=of s_0, label=above:$\textcolor{blue}{\{r_1\}}$]  {$s_1$};
			\node[state] (s_2) [below right=of s_0, label=below:$\textcolor{blue}{\{r_2\}}$]  {$s_2$};
			\path[->] 
			(s_0) edge [bend left=15] node {$\textcolor{red}{a_{f1}}$} (s_1)
			(s_1) edge [bend left=15] node {$\textcolor{red}{a_{1f}}$} (s_0)
			(s_1) edge [loop right] node {$\textcolor{red}{a_{11}}$} (s_1)
			(s_2) edge [loop right] node {$\textcolor{red}{a_{22}}$} (s_2)
			(s_0) edge [bend left=15] node {$\textcolor{red}{a_{f2}}$} (s_2)
			(s_2) edge [bend left=15] node {$\textcolor{red}{a_{2f}}$} (s_0);
			\end{tikzpicture}
			}
			}}
	$\qquad$
	\subfloat[]{{
		\scalebox{.8}{
			\begin{tikzpicture}[shorten >=1pt,node distance=3cm,on grid,auto] 
			\node[state,initial] (s_0)   {$(s_0,q_0)$}; 
			\node[state,accepting] (s_1) [above right=of s_0, label=above:$\textcolor{ForestGreen}{F_1}$]  {$(s_1,q_1)$};
			\node[state] (s_2) [below right=of s_0]  {$(s_0,q_0)$};
			\draw[blue, very thick] (1.2,-3) rectangle (3,-1.2);
			\path (3.65,0) node(x) {\textcolor{blue}{Cycle1}};
			\path (3.65,-2.1) node(x) {\textcolor{blue}{Cycle2}};
			\node[state,accepting] (s_3) [right=of s_1, label=above:$\textcolor{ForestGreen}{F_1}$]  {$(s_1,q_1)$};
			\node[state,accepting] (s_4) [right=of s_2, label=below:$\textcolor{ForestGreen}{F_2}$]  {$(s_2,q_2)$};
			\path[->] 
			(s_0) edge node {$\textcolor{red}{a_{f1}^p}$} (s_1)
			(s_1) edge [bend right=15] node {$\textcolor{red}{a_{1f}^p}$} (s_2)
			(s_2) edge [bend left=15] node {$\textcolor{red}{a_{f1}^p}$} (s_3)
			(s_3) edge [bend left=15] node {$\textcolor{red}{a_{1f}^p}$} (s_2)
			(s_3) edge [loop right] node {$\textcolor{red}{a_{11}^p}$} (s_3)
			(s_2) edge [bend left=15] node {$\textcolor{red}{a_{f2}^p}$} (s_4)
			(s_4) edge [bend left=15] node {$\textcolor{red}{a_{2f}^p}$} (s_2)
			(s_4) edge [loop right] node {$\textcolor{red}{a_{22}^p}$} (s_4)
			;
			\end{tikzpicture}
			}
		}}
		\caption{\label{fig:automaton} (a) LDGBA of LTL formula
		$\varphi_{B1}$. (b) Generalized MDP model for
		Ball-pass case. (c) The standard product MDP.}
	\end{figure}
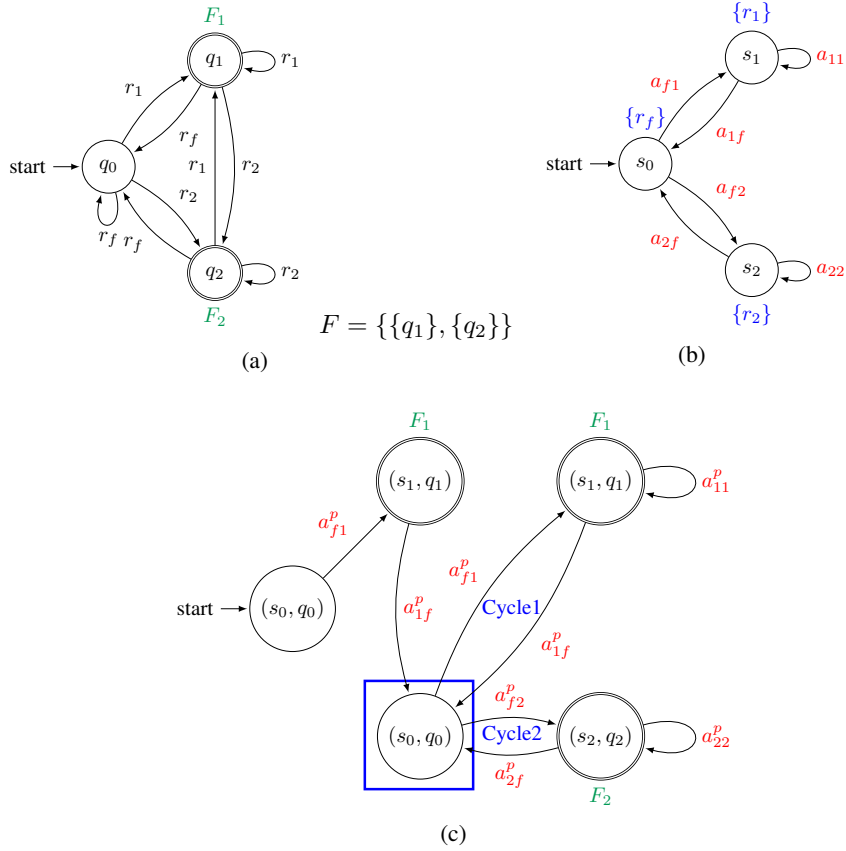


\textbf{(1). Ball-Pass}: We first demonstrate the developed control synthesis in a Ball-Pass
environment ($600m\times600m$). Consider a red ball moving according
to the following dynamics: $\ddot{x}=a_{x}$, $\ddot{y}=a_{y}+g$
in Fig. \ref{fig:case_study3} (a), where $\left(x,y\right)$
is the planar position of the ball, $a_{x},a_{y}$ represent accelerations along $x$
and $y$ induced by an external force (i.e., the control input),
respectively, and $g$ is the gravitational acceleration. The simulation step size is
$\varDelta t=0.05s$, and the acceleration range is $a_{x},a_{y}\in\left[0,1\right]$
$m/s$. We consider two LTL tasks, e.g., $\varphi_{B1}=\left(\oblong\lozenge\mathtt{Region}1\right)\land\left(\oblong\lozenge\mathtt{Region}2\right)$ and $\varphi_{B2}=\lozenge\left(\mathtt{\mathtt{Region}1}\wedge \lozenge\mathtt{\mathtt{Region}2}\right)$. $\varphi_{B1}$ is is a task over the infinite horizon and requires the agent repetitively visits $\mathtt{Region}1$ and $\mathtt{Region}2$, and
$\varphi_{B2}$ is a task over the finite horizon and requires the agent first to visit $\mathtt{Region}1$
and then $\mathtt{Region}2$.

\begin{figure}
	\centering{}\includegraphics[scale=0.45]{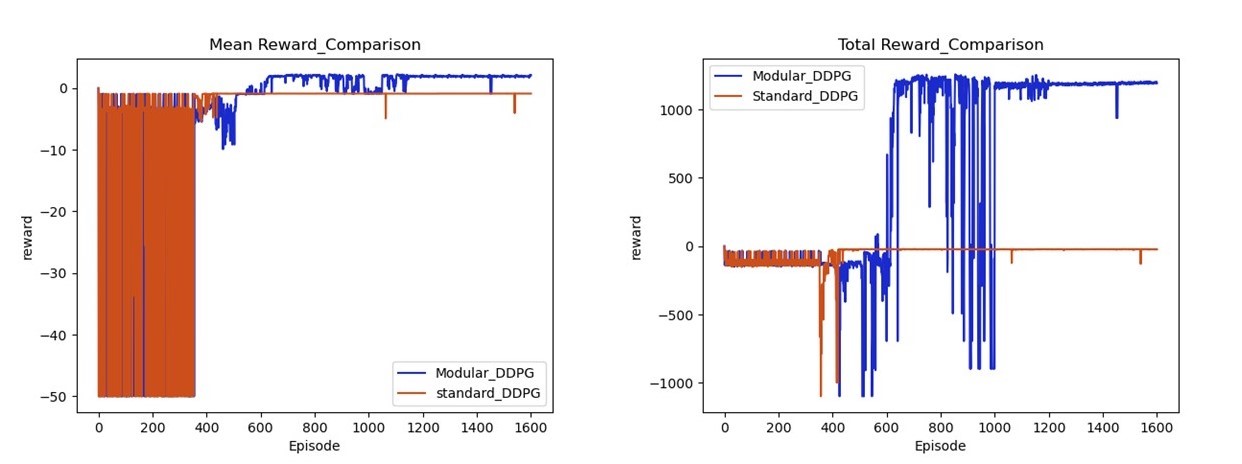}\caption{\label{fig:R_ball-pass_task1} The evolution of reward for $\varphi_{B1}$. (a) Average reward. (b) Total reward.}
\end{figure}

\begin{figure}
	\centering{}\includegraphics[scale=0.45]{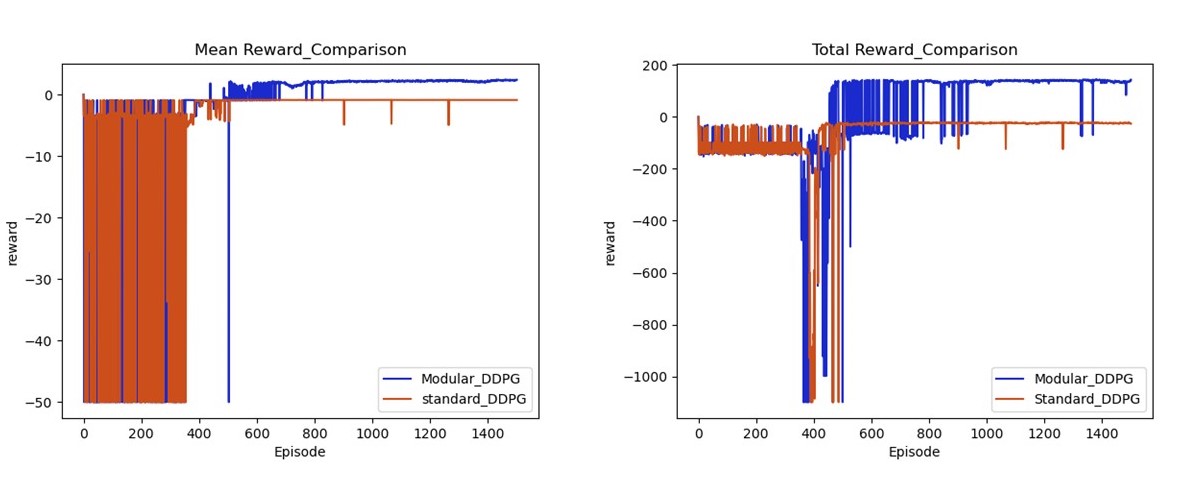}\caption{\label{fig:R_ball-pass_task2} The evolution of reward for $\varphi_{B2}$. (a) Average reward. (b) Total reward.}
\end{figure}
For comparison of method (i), Fig. \ref{fig:automaton} (a) shows the corresponding
LDGBA of $\varphi_{B1}$ with two accepting sets $F=\left\{ \left\{ q_{1}\right\} ,\left\{ q_{2}\right\} \right\} $, and $r_1$, $r_2$ and $r_f$ represent $\mathtt{Region}$ 1, $\mathtt{Region}$ 2 and free space being visited, respectively. Fig. \ref{fig:automaton} (b) shows the Ball-pass MDP model, where
$a_{fi}$ (or $a_{if}$ ) represent a sequence of continuous actions
that drive the ball from free space to $\mathtt{Region}$ $i$ (or
inverse). In addition, Fig. \ref{fig:automaton} (c) shows the resulting standard
product MDP. By  Definition \ref{def:LDGBA}, the policy that satisfies $\varphi_{B1}$
should enforce the repetitive trajectories, i.e., cycle 1 and cycle 2
in Fig. \ref{fig:automaton} (c). However, there exists no deterministic
policy that can periodically select two actions $a_{f1}^{P}$ and
$a_{f2}^{P}$ at state $\left(s_{f},q_{0}\right)$ (marked with a
blue rectangle) in Fig. \ref{fig:automaton} (c). As a result, applying
standard product MDP cannot generate a pure deterministic optimal policy to complete task $\varphi_{B1}$. Such a scenario also happens for task $\varphi_{C1}$. In contrast, the tracking-frontier
set of EP-MDP developed in this work can resolve this issue by recording
unvisited acceptance sets and being embedded with each state at every time-step
via one-hot encoding. We simulate 10 runs for tasks $\varphi_{B1}$ and $\varphi_{C1}$ and select the worst case of applying standard product MDP as comparison. For comparison of method (ii), we conduct $1600$
episodes for each task ($\varphi_{B1}$ and $\varphi_{B2}$) of Ball-pass, and the reward collections
are shown in Fig. \ref{fig:R_ball-pass_task1} and \ref{fig:R_ball-pass_task2}..

\textbf{(2). CartPole}: We also test our control framework for the Cart-pole\footnote{\url{https://gym.openai.com/envs/CartPole-v0/\#barto83}}
in Fig. \ref{fig:case_study3} (b). The pendulum starts upright
with an initial angle between $-0.05$ and $0.05$ rads. The horizontal
force exerted on the cart is defined over a continuous space (action-space)
with a range $\left(-10N,10N\right)$. The green and yellow regions
range from $-1.44$ to $-0.96$ m and from $0.96$ to $1.44$ m, respectively.
The objective is to prevent the pendulum from falling over while moving
the cart between the yellow and green regions. Similarly, we consider two tasks
$\varphi_{C1}=\left(\oblong\lozenge\mathtt{Green}\right)\land\left(\oblong\lozenge\mathtt{Yellow}\right)\land\lnot\oblong\mathtt{Unsafe}$
and $\varphi_{C2}=\lozenge\left(\mathtt{\mathtt{Green}}\wedge\lozenge\mathtt{Yellow}\right)\land\lnot\oblong\mathtt{Unsafe}$,
where $\mathtt{Unsafe}$ represents the condition that the pendulum
falls over or exceeds the red line, and $\mathtt{Green},\mathtt{Yellow}$
represent colored areas. Method (i) has the similar issues as shown in Fig. for task $\varphi_{C1}$ over infinite horizon. To compare the method (ii), we conduct mean reward collections over $1500$ episodes for each task ($\varphi_{C1}$ and $\varphi_{C2}$) of CartPole shown in Fig. \ref{fig:Cart-Pole}. We can conclude the modular DDPG has a better performance of collecting positive rewards during learning. The figures also illustrate that the modular DDPG has a better performance than the standard DDPG for tasks over both finite and infinite horizons. We simulate 10 runs and select the worst case of applying standard product MDP as comparison that can be found in our Github repository. 

\begin{figure}
	\centering{}\includegraphics[scale=0.45]{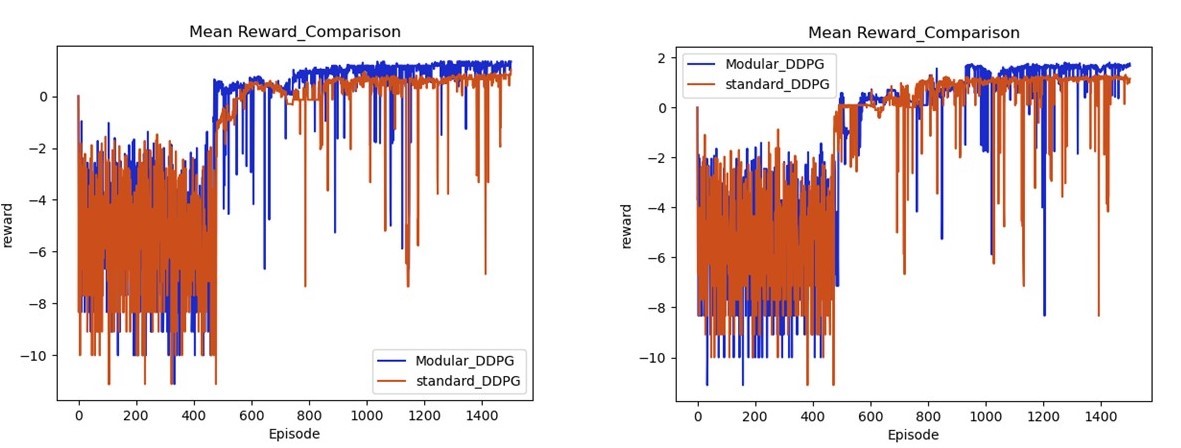}\caption{\label{fig:Cart-Pole} The evolution of average reward. (a) Task $\varphi_{C1}$. (b) Task $\varphi_{C2}$.}
\end{figure}

\subsection{Image-based Environment \label{subsec:Large_scale}}

\begin{figure}
	\centering{}\includegraphics[scale=0.45]{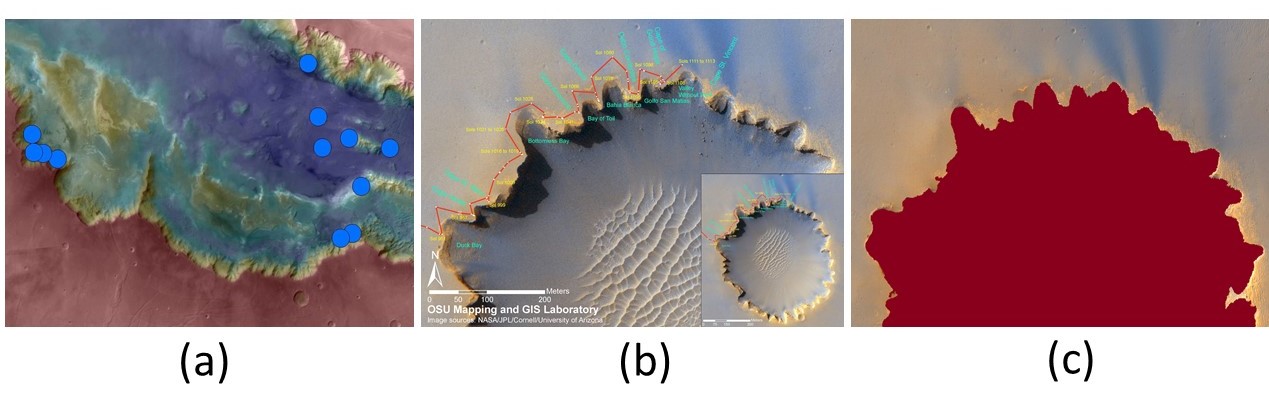}\caption{\label{fig:Mars_environment} Mars exploration mission. (a) Melas
		Chasma with possible location of water (blue dots). (b) Victoria Crater
		with exploration path (red line). (c) Processed image of Victoria
		Crater.}
\end{figure}

In this section, we test our algorithm in large scale continuous environments, and conduct motion planning to complete two Mars exploration missions using
satellite images  as shown in Fig. \ref{fig:Mars_environment}. The missions
are to explore areas around the Melas Chasma \cite{Mcewen2014} and
the Victoria Crater \cite{Squyres2009}. In the Melas Chasma of Fig.
\ref{fig:Mars_environment} (a), there are a number of spots with the potential
presence of water, possible river valleys and lakes. According
to NASA, the blue spots are possible locations of water. The LTL task,
in this case, is to visit both clusters by visiting any blue dot in each cluster while avoiding unsafe (red land) regions in Fig.
\ref{fig:Mars_environment} (a). The Victoria Crater in Fig. \ref{fig:Mars_environment}
(b) is an impact crater located near the equator of Mars. Layered
sedimentary rocks are exposed along the wall of crater, providing
information about the ancient surface condition of Mars. The mission
is related to visiting all spots along with the path of the well-known Mars
Rover Opportunity that are given in Fig. \ref{fig:Mars_environment}
(b), and avoiding the unsafe areas (red regions in Fig. \ref{fig:Mars_environment}
(c)). We employ LTL to specify such missions. The LTL specifications
for Melas Chasma and Victoria Crater are expressed as following separately:
\[
\varphi_{Melas}=\oblong\lozenge\mathtt{M_{1}}\land\oblong\lozenge\mathtt{M_{2}}\land\lnot\oblong\mathtt{M_{unsafe}},
\]
\[
\varphi_{Victoria}=\oblong\lozenge V_{1}\land\oblong\lozenge V_{2}\ldots\oblong\lozenge V_{12}\land\lnot\oblong\mathtt{\mathtt{V_{unsafe}}},
\]
where $\mathtt{M_{i}}$ represents $i$-th target and $\mathtt{\mathtt{M_{unsafe}}}$
indicates unsafe areas in Melas Chasma. The task $\varphi_{Victoria}$
in Victoria Crater has similar settings. At each stage, the rover
agent has a continuous action-range $\left[0,2\pi\right)$ and the
resulting outcome of each action is to move toward the direction of
the action within a range drawn from $\left(0,D\right]$. The dimensions
and action ranges of Fig. \ref{fig:Mars_environment} (b) and Fig.
\ref{fig:Mars_environment} (b) are $456km\times322km$, $D=2km$ and
$746m\times530m$, $D=10m$, respectively.

The LDGBA for $\varphi_{Melas}$ has $3$ states and $2$ accepting
sets, and the one for $\varphi_{Victoria}$ has $13$ states and $12$
accepting sets. The scenario is to train a modular DDPG with the tuning
reward design described in Section \ref{subsec:RL-reward}, which
can autonomously satisfy the complex tasks by accessing the images.

\begin{figure}
	\centering{}\includegraphics[scale=0.28]{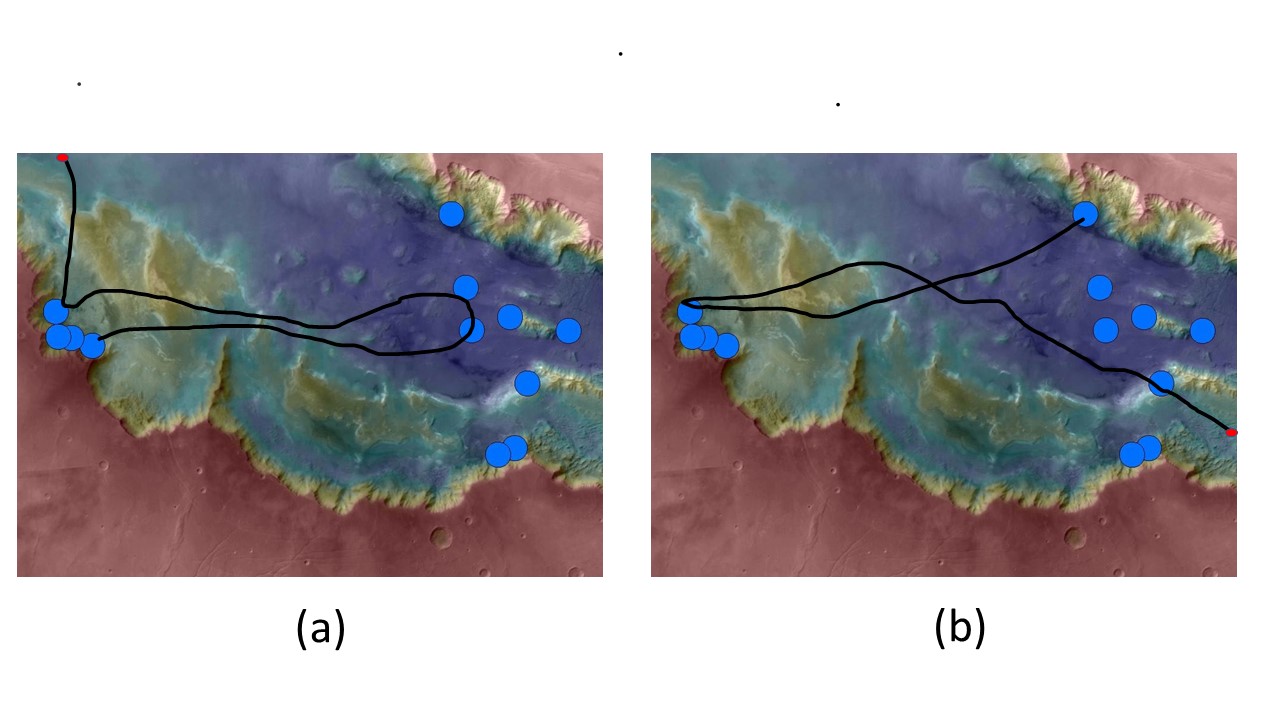}\caption{\label{fig: Melas_traj.} The trajectories with different
		initial locations (image 
		coordinates). (a) Top-side with initial coordinates (2,18). (b) Right-side with initial coordinate (105,199).}
\end{figure}

\begin{figure}
	\centering{}\includegraphics[scale=0.3]{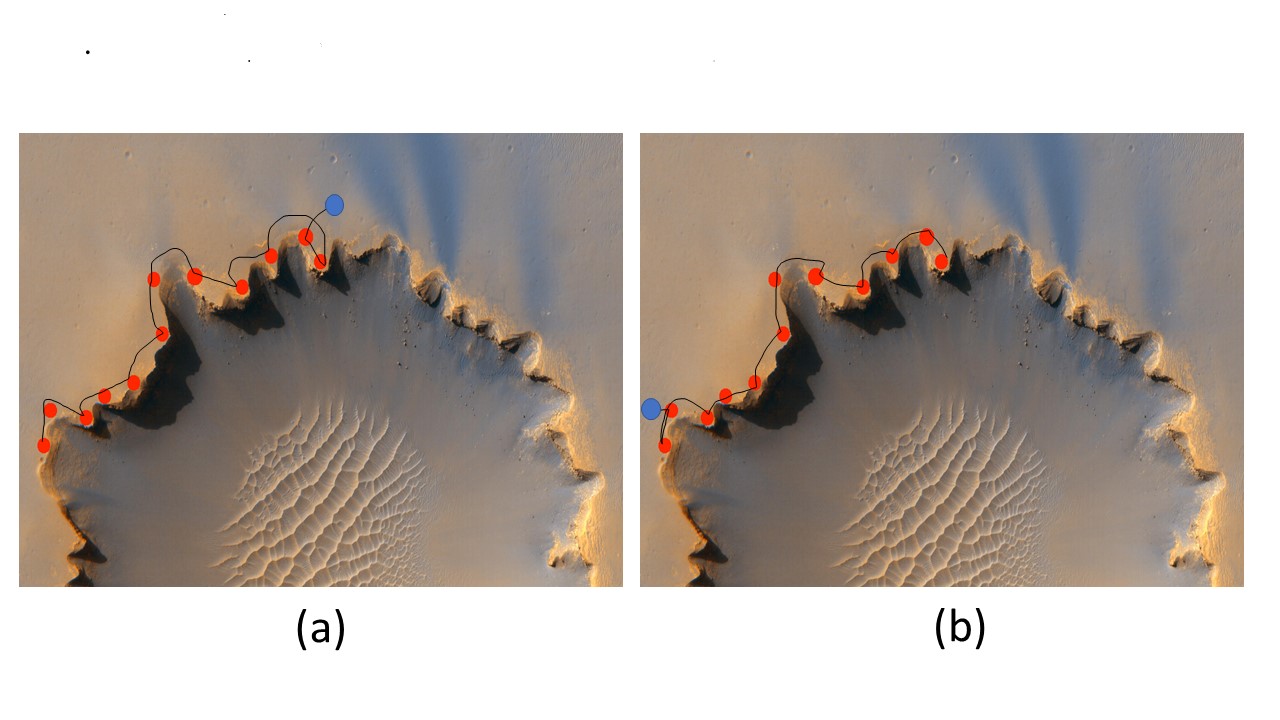}\caption{\label{fig: Victorria_traj.} The trajectories with different
		initial locations (marked as blue cycles). (a) top-right. (b) bottom-left}
\end{figure}

The surveillance task $\varphi_{Melas}$ requires visiting right and
left clusters in Fig. \ref{fig:Mars_environment} (b), and the similar
task $\varphi_{Victoria}$ requires visiting all the spots associated
with the path in Fig. \ref{fig:Mars_environment} (b). After training, we show
the trajectories of one round satisfaction with different initial locations for
the two tasks in Fig. \ref{fig: Melas_traj.} and \ref{fig: Victorria_traj.}.
The training time is influenced by the number of actor-critic neural
network pairs.

\subsection{Results Summary \label{subsec:summary}}
Based on the experimental results, we can summarize our contributions compared with two other methods: (i) the modular DDPG with standard product MDP, (ii) the standard DDPG with EP-MDP.
 
As analyzed in Fig. \ref{fig:automaton}, the modular DDPG with standard product MDP can not be guaranteed to synthesise pure deterministic policies for a number of LTL properties, e.g., $\varphi_{B1}$, $\varphi_{C1}$, $\varphi_{Melas}$ and $\varphi_{Victoria}$. The reason is that the corresponding LDGBA for each of them has multiple accepting sets, and there exist no deterministic policies for one state of $\mathcal{P}$ to select different actions.

\begin{table}
	\caption{\label{tab:success_rate}Comparison of standard and modular DDPG with
		EP-MDP. Statistics are taken over 200 runs.}
	\centering{}\resizebox{0.4\textwidth}{!}{
		\begin{tabular}{c|cc}
			\hline 
			LTL Task & DDPG & Success rate\tabularnewline
			\hline 
			\multirow{2}{*}{$\varphi_{B1}$} & \multicolumn{1}{c}{Modular DDPG} & $100\%$\tabularnewline
			\cline{2-3} \cline{3-3} 
			& Standard DDPG & $0\%$\tabularnewline
			\hline 
			\multirow{2}{*}{$\varphi_{B2}$} & Modular DDPG & $100\%$\tabularnewline
			\cline{2-3} \cline{3-3} 
			& Standard DDPG & $77.5\%$\tabularnewline
			\hline 
			\multirow{2}{*}{$\varphi_{C1}$} & Modular DDPG & $100\%$\tabularnewline
			\cline{2-3} \cline{3-3} 
			& Standard DDPG & $23.5\%$\tabularnewline
			\hline 
			\multirow{2}{*}{$\varphi_{C2}$} & Modular DDPG & $100\%$\tabularnewline
			\cline{2-3} \cline{3-3} 
			& Standard DDPG & $61\%$\tabularnewline
			\hline 
			\multirow{2}{*}{$\varphi_{Melas}$} & Modular DDPG & $100\%$\tabularnewline
			\cline{2-3} \cline{3-3} 
			& Standard DDPG & $0\%$\tabularnewline
			\hline 
			\multirow{2}{*}{$\varphi_{Victoria}$} & Modular DDPG & $100\%$\tabularnewline
			\cline{2-3} \cline{3-3} 
			& Standard DDPG & $0\%$\tabularnewline
			\hline 
	\end{tabular}}
\end{table}

Then, we take $200$ runs for all tasks above and analyze the success rate for all aforementioned tasks compared with method (ii) in Table \ref{tab:success_rate}. In practice, since the training process of a deep neural network has limited steps for each episode and finite number of total episodes, and dimensions of automaton structure grows for more complex tasks, it becomes difficult for standard DDPG to explore the whole tasks especially for tasks over infinite horizon. We can conclude
that the standard DDPG with a recording method as \cite{Wang2020} might yield poor performance results for tasks with repetitive pattern
(infinite horizon), whereas the modular DDPG has better performance
from the perspective of training and success rate. It also shall be noted that the surveillance tasks were mainly considered in this paper. However, the temporal operators in LTL, such as "next" and "until", enable the modular DDPG algorithm to solve problems with other types of complex tasks, e.g., fairness and safety \cite{Hasanbeig2019,cautiousRL}. For instance, consider the following LTL formula
$$\varphi=a\wedge\bigcirc(\lozenge\square a\vee\lozenge\square b)$$   

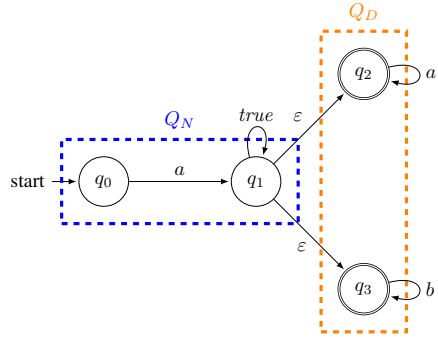
\begin{figure}[!t]\centering
	\scalebox{.75}{
		\begin{tikzpicture}[shorten >=1pt,node distance=2.7cm,on grid,auto] 
		\node[state,initial] (q_0)   {$q_0$}; 
		\node[state] (q_1) [right=of q_0]  {$q_1$}; 
		\node[state,accepting] (q_2) [above right=of q_1] {$q_2$}; 
		\node[state,accepting] (q_3) [below right=of q_1] {$q_3$};
		\node (N)[draw=blue, fit= (q_0)(q_1),dashed,ultra thick,inner sep=3mm] {};
		\node [yshift=2ex, blue] at (N.north) {${Q}_N$};
		\node (D)[draw=orange, fit= (q_2)(q_3),dashed,ultra thick,inner sep=3mm] {};
		\node [yshift=2ex, orange] at (D.north) {${Q}_D$}; 
		\path[->] 
		(q_0) edge node {$a$} (q_1)
		(q_1) edge  node {$\varepsilon$} (q_2)
		(q_1) edge [loop above] node {$\mathit{true}$} (q_1)
		(q_1) edge [below] node {$\varepsilon~~~$} (q_3)
		(q_2) edge  [loop right] node {$a$} (q_2)
		(q_3) edge  [loop right] node {$b$} (q_3);
		\end{tikzpicture}}
	\caption{LDGBA for the formula  $a\wedge\bigcirc(\lozenge\square a\vee\lozenge\square b)$.}
	\label{fig:ldgba_ex}  
\end{figure}

In those cases, LTL formulas can be converted to LDGBAs \cite{Sickert2016}, as in Fig.~\ref{fig:ldgba_ex}, and the same framework is applied.

At last, although several actor-critic neural network pairs are adopted in the modular architecture, they are synchronously trained online, and each of them is only responsible for a sub-task. This setup is more effective to complete LTL complex tasks. The training time for each task with two different algorithms (the standard DDPG and the modular DDPG) is shown in Table \ref{tab:training}. It can be seen that the runtime complexity is not increased in the modular DDPG.
\begin{table}
	\caption{\label{tab:training}Training time analysis of standard and modular DDPG with
		EP-MDP.}
	\centering{}%
	\begin{tabular}{|ccc|cc|c}
	    \hline
	    \multicolumn{3}{|c|}{Tasks and Training Parameters} & \multicolumn{2}{|c|}{Training Time (minute, hour, day)} \\ 
		\hline 
		LTL Task & Maximum steps & Episode & Standard DDPG & Modular DDPG  \\ 
		\hline
		$\varphi_{B1}$ & $200$ & $1600$ & $11.3$ min & $12.1$ min  \\
		
		$\varphi_{B2}$ & $500$ & $1500$ & $17.0$ min & $20.4$ min \tabularnewline
		
		$\varphi_{C1}$ & $200$ & $1500$ & $14.1$ min & $13.6$ min \tabularnewline
		
		$\varphi_{C2}$ & $500$ & $1500$ &$20.2$ min & $21.0$ min \tabularnewline
		
		$\varphi_{Melas}$ & $2000$ & $10000$ & $5.5$ hr & $5.4$ hr \tabularnewline
		
		$\varphi_{Victoria}$ & $8000$ & $50000$ & $48.0$ hr & $50.0$ hr \tabularnewline
	
		\hline 
	\end{tabular}
\end{table}

\section{Conclusions}

In this paper, a model-free deep RL learning is developed to synthesize
control policies in continuous-state and continuous-action MDPs. The designed EP-MDP can enforce the task satisfaction, and a modular
DDPG framework is proposed to decompose the temporal automaton into
interrelated sub-tasks such that the resulted optimal policies are
shown to satisfy the LTL specifications with a higher success rate compared
to the standard DDPG. Based on the comparisons, we demonstrate the benefits of applying LDGBA and modular DDPG to improve the learning performance and enforce the task satisfaction with high success rates.

\section{Appendix}

\subsection{Proof of Theorem~\ref{lemma:probability} \label{append:1}}
Based on whether or not the path $\boldsymbol{x}_{t}$ intersects
with accepting states of $F_{i}^{\mathcal{\mathcal{P}}}$, the expected
return in (\ref{eq:ExpRetrn}) can be rewritten as 
\begin{equation}
\begin{aligned}U^{\boldsymbol{\pi}}\left(x\right)= & \mathbb{E}^{\boldsymbol{\pi}}\left[\mathcal{D}\left(\boldsymbol{x}_{t}\right)\left|\boldsymbol{x}_{t}\models\diamondsuit F_{i}^{\mathcal{\mathcal{P}}}\right.\right]\cdot\Pr{}^{\boldsymbol{\pi}}\left[x\mid=\diamondsuit F_{i}^{\mathcal{\mathcal{P}}}\right]\\
 & +\mathbb{E}^{\boldsymbol{\pi}}\left[\mathcal{D}\left(\boldsymbol{x}_{t}\right)\left|\boldsymbol{x}_{t}\neq\diamondsuit F_{i}^{\mathcal{\mathcal{P}}}\right.\right]\cdot\Pr{}^{\boldsymbol{\pi}}\left[x\mid\neq\diamondsuit F_{i}^{\mathcal{\mathcal{P}}}\right]
\end{aligned}
\label{eq:proof 1}
\end{equation}
where $\Pr^{\boldsymbol{\pi}}\left[x\models\diamondsuit F_{i}^{\mathcal{\mathcal{P}}}\right]$
and $\Pr^{\boldsymbol{\pi}}\left[x\mid\neq\diamondsuit F_{i}^{\mathcal{\mathcal{P}}}\right]$
represent the probability of eventually reaching and not reaching $F_{i}^{\mathcal{\mathcal{P}}}$
eventually under policy $\pi$ starting from state $x$, respectively.

To find the lower bound of $U^{\boldsymbol{\pi}}\left(x\right)$,
for any $\boldsymbol{x}_{t}$ with $\boldsymbol{x}_{t}\left[t\right]=x$,
let $t+N_{t}$ be the index that $\boldsymbol{x}_{t}$ first intersects
a state in $X_{\mathcal{\mathcal{\mathcal{P}}}}^{\boldsymbol{\pi}}$,
i.e., $N_{t}=\min\left[i\bigl|\boldsymbol{x}_{t}\left[t+i\right]\in X_{\mathcal{\mathcal{\mathcal{P}}}}^{\boldsymbol{\pi}}\right]$.
The following holds
\begin{equation}
\begin{aligned} & \mathbb{E}^{\boldsymbol{\pi}}\left[\mathcal{D}\left(\boldsymbol{x}_{t}\right)\left|\boldsymbol{x}_{t}\models\diamondsuit F_{i}^{\mathcal{\mathcal{P}}}\right.\right]\\
 & \overset{_{^{\left(1\right)}}}{\geq}\mathbb{E}^{\boldsymbol{\pi}}\left[\mathcal{D}\left(\boldsymbol{x}_{t}\right)\left|\boldsymbol{x}_{t}\cap X_{\mathcal{\mathcal{\mathcal{P}}}}^{\boldsymbol{\pi}}\neq\emptyset\right.\right]\\
 & \overset{_{^{\left(2\right)}}}{\geq}\mathbb{E}^{\boldsymbol{\pi}}\left[\gamma_{F}^{N_{t}}\cdot\mathcal{D}\left(\boldsymbol{x}_{t}\left[t+N_{t}:\right]\right)\left|\boldsymbol{x}_{t}\left[t+N_{t}\right]=x\right|\boldsymbol{x}_{t}\cap X_{\mathcal{\mathcal{\mathcal{P}}}}^{\boldsymbol{\pi}}\neq\emptyset\Bigr]\right.\\
 & \overset{_{^{\left(3\right)}}}{\geq}\mathbb{E}^{\boldsymbol{\pi}}\left[\gamma_{F}^{N_{t}}\Bigl|\right.\boldsymbol{x}_{t}\cap X_{\mathcal{\mathcal{\mathcal{P}}}}^{\boldsymbol{\pi}}\neq\emptyset\Bigr]\cdot U_{\min}^{\boldsymbol{\pi}}\left(\boldsymbol{x}_{t}\left[t+N_{t}\right]\right)\\
 & \overset{_{^{\left(4\right)}}}{\geq}\gamma_{F}^{\mathbb{E}^{\boldsymbol{\pi}}\left[N_{t}\left|\boldsymbol{x}_{t}\left[t\right]=x\right|\boldsymbol{x}_{t}\cap X_{\mathcal{\mathcal{\mathcal{P}}}}^{\boldsymbol{\pi}}\neq\emptyset\right]}\cdot U_{\min}^{\boldsymbol{\pi}}\left(x_{Acc}\right)\\
 & =\gamma_{F}^{n_{t}}\cdot U_{\min}^{\boldsymbol{\pi}}\left(x_{Acc}\right),
\end{aligned}
\label{eq:proof 2}
\end{equation}

where $x_{Acc}\in X_{\mathcal{\mathcal{\mathcal{P}}}}^{\boldsymbol{\pi}}$,
$U_{\min}^{\boldsymbol{\pi}}\left(x_{Acc}\right)=\min_{x\in X_{\mathcal{\mathcal{\mathcal{P}}}}^{\boldsymbol{\pi}}}U^{\boldsymbol{\pi}}\left(x\right)$,
and $n_{t}$ is a constant. By Lemma \ref{lemma:1}, one has $\underset{\gamma_{F}\shortrightarrow1^{-}}{\lim}U_{\min}^{\boldsymbol{\pi}}\left(x_{Acc}\right)=~1$.
In (\ref{eq:proof 2}), the first inequality \textbf{(1)} holds because visiting $X_{\mathcal{\mathcal{\mathcal{P}}}}^{\boldsymbol{\pi}}$
is one of the cases for $\diamondsuit F_{i}^{\mathcal{\mathcal{P}}}$ that satisfy $\boldsymbol{x}_{t}\models\diamondsuit F_{i}^{\mathcal{\mathcal{P}}}$, e.g., $F_{i}^{\mathcal{\mathcal{P}}}$ can be placed outside of all BSCCs;
the second inequality \textbf{(2)} holds due to Lemma \ref{lemma:1}; the third
inequality \textbf{(3)} holds due to the Markov properties of (\ref{eq:DisctRetrn})
and (\ref{eq:ExpRetrn}); the fourth inequality \textbf{(4)} holds due to Jensen's
inequality. Based on (\ref{eq:proof 2}), the lower bound of (\ref{eq:proof 1})
is $U^{\boldsymbol{\pi}}\left(x\right)\geq\gamma_{F}^{n_{t}}\cdot U_{\min}^{\boldsymbol{\pi}}\left(x_{Acc}\right)\cdot\Pr{}^{\boldsymbol{\pi}}\left[x\models\diamondsuit F_{i}^{\mathcal{\mathcal{P}}}\right]$
from which one has 
\begin{equation}
\underset{\gamma_{F}\shortrightarrow1^{-}}{\lim}U^{\boldsymbol{\pi}}\left(x\right)\geq\gamma_{F}^{n_{t}}\cdot\Pr{}^{\boldsymbol{\pi}}\left[x\models\diamondsuit F_{i}^{\mathcal{\mathcal{P}}}\right].\label{eq:lower bound}
\end{equation}

Similarly, let $t+M_{t}$ denote the index that $\boldsymbol{x}_{t}$
first enters the BSCC that contains no accepting states. We have
\begin{equation}
\begin{array}{c}
\mathbb{E}^{\boldsymbol{\pi}}\left[\mathcal{D}\left(\boldsymbol{x}_{t}\right)\left|\boldsymbol{x}_{t}\mid\neq\diamondsuit F_{i}^{\mathcal{\mathcal{P}}}\right.\right]\overset{_{^{\left(1\right)}}}{\leq}\mathbb{E}^{\pi}\left[1-r_{F}^{M_{t}}\left|\boldsymbol{x}_{t}\mid\neq\diamondsuit F_{i}^{\mathcal{\mathcal{P}}}\right.\right]\\
\overset{_{^{\left(2\right)}}}{\leq}1-r_{F}^{\mathbb{E}^{\boldsymbol{\pi}}\left[M_{t}\left|\boldsymbol{x}_{t}\left[t\right]=x,\right.\boldsymbol{x}_{t}\mid\neq\diamondsuit F^{\mathcal{P}}\right]}=1-r_{F}^{m_{t}}
\end{array}\label{eq:proof 3}
\end{equation}
where $m_{t}$ is a constant and (\ref{eq:proof 3}) holds due to
Lemma \ref{lemma:1} and Markov properties.

Hence, the upper bound of (\ref{eq:proof 1}) is obtained as 
\begin{equation}
\underset{\gamma_{F}\shortrightarrow1^{-}}{\lim}U^{\boldsymbol{\pi}}\left(x\right)\leq\Pr{}^{\boldsymbol{\pi}}\left[x\models\diamondsuit F_{i}^{\mathcal{\mathcal{P}}}\right]+\left(1-r_{F}^{m_{t}}\right)\Pr{}^{\boldsymbol{\pi}}\left[x\mid\neq\diamondsuit F_{i}^{\mathcal{\mathcal{P}}}\right].\label{eq:upper bound}
\end{equation}
By (\ref{eq:lower bound}) and (\ref{eq:upper bound}), we can conclude
\[
\begin{array}{c}
\gamma_{F}^{n_{t}}\cdot\Pr{}^{\boldsymbol{\pi}}\left[x\models\diamondsuit F_{i}^{\mathcal{\mathcal{P}}}\right]\leq\underset{\gamma_{F}\shortrightarrow1^{-}}{\lim}U^{\boldsymbol{\pi}}\left(x\right)\\
\leq\Pr{}^{\boldsymbol{\pi}}\left[x\models\diamondsuit F_{i}^{\mathcal{\mathcal{P}}}\right]+\left(1-r_{F}^{m_{t}}\right)\cdot\Pr{}^{\boldsymbol{\pi}}\left[x\mid\neq\diamondsuit F_{i}^{\mathcal{\mathcal{P}}}\right]
\end{array}
\]
According to $\underset{\gamma_{F}\shortrightarrow1^{-}}{\lim}r_{F}\left(\gamma_{F}\right)=1$
in the reward function, (\ref{eq:reachability}) can be concluded. 

\subsection{Proof of Theorem~\ref{thm2} \label{Append:2}}

For any policy $\boldsymbol{\pi}$, $MC_{\mathcal{\mathcal{{\mathcal{P}}}}}^{\boldsymbol{\pi}}=\ensuremath{\mathcal{T}_{\boldsymbol{\pi}}}\sqcup\ensuremath{\mathcal{R}_{\boldsymbol{\pi}}^{1}\sqcup\ensuremath{\mathcal{R}_{\boldsymbol{\pi}}}^{2}\ldots\ensuremath{\mathcal{R}_{\boldsymbol{\pi}}}^{n_{R}}}.$
Let $\boldsymbol{U}_{\boldsymbol{\pi}}=\left[\begin{array}{ccc}
U^{\boldsymbol{\pi}}\left(x_{0}\right) & U^{\boldsymbol{\pi}}\left(x_{1}\right) & \ldots\end{array}\right]^{T}\in\mathbb{R}^{\left|X\right|}$ denote the stacked expected return under policy $\pi$, which can
be reorganized as 
\begin{equation}
\begin{aligned}\left[\begin{array}{c}
\boldsymbol{U}_{\boldsymbol{\pi}}^{tr}\\
\boldsymbol{U}_{\boldsymbol{\pi}}^{rec}
\end{array}\right]= & \stackrel[n=0]{\infty}{\sum}\left(\stackrel[j=0]{n-1}{\prod}\left[\begin{array}{cc}
\boldsymbol{\boldsymbol{\gamma}}_{\boldsymbol{\pi}}^{\ensuremath{\mathcal{T}}} & \boldsymbol{\boldsymbol{\gamma}}_{\boldsymbol{\pi}}^{\ensuremath{tr}}\\
\boldsymbol{0}_{\sum_{i=1}^{m}N_{i}\times r} & \boldsymbol{\boldsymbol{\gamma}}_{\boldsymbol{\pi}}^{rec}
\end{array}\right]\right)\\
 & \cdot\left[\begin{array}{cc}
\boldsymbol{P}_{\boldsymbol{\pi}}\left(\ensuremath{\mathcal{T}},\ensuremath{\mathcal{T}}\right) & \boldsymbol{P}_{\boldsymbol{\pi}}^{tr}\\
\boldsymbol{0}_{\sum_{i=1}^{m}N_{i}\times r} & \boldsymbol{P}_{\boldsymbol{\pi}}\left(\mathcal{R},\mathcal{R}\right)
\end{array}\right]^{n}\left[\begin{array}{c}
\boldsymbol{R}_{\boldsymbol{\pi}}^{tr}\\
\boldsymbol{R}_{\boldsymbol{\pi}}^{rec}
\end{array}\right],
\end{aligned}
\label{eq: utility_function}
\end{equation}
where $\boldsymbol{U}_{\boldsymbol{\pi}}^{tr}$ and $\boldsymbol{U}_{\boldsymbol{\pi}}^{rec}$
are the expected return of states in transient and recurrent classes
under policy $\boldsymbol{\pi}$, respectively. In (\ref{eq: utility_function}),
$\boldsymbol{P}_{\boldsymbol{\pi}}\left(\ensuremath{\mathcal{T}},\ensuremath{\mathcal{T}}\right)\in\mathbb{R}^{r\times r}$
is the probability transition matrix between states in $\ensuremath{\mathcal{T}_{\boldsymbol{\pi}}}$,
and $\boldsymbol{P}_{\boldsymbol{\pi}}^{tr}=\left[P_{\boldsymbol{\pi}}^{tr_{1}}\ldots P_{\boldsymbol{\pi}}^{tr_{m}}\right]\in\mathbb{R}^{r\times\sum_{i=1}^{m}N_{i}}$
is the probability transition matrix where $P_{\boldsymbol{\pi}}^{tr_{i}}\mathbb{\in R}^{r\times N_{i}}$
represents the transition probability from a transient state in $\ensuremath{\mathcal{T}_{\boldsymbol{\pi}}}$
to a state of $\mathcal{R}_{\boldsymbol{\pi}}^{i}$. The $\boldsymbol{P}_{\boldsymbol{\pi}}\left(\mathcal{R},\mathcal{R}\right)$
is a diagonal block matrix, where the $i$th block is a $N_{i}\times N_{i}$
matrix containing transition probabilities between states within $\mathcal{R}_{\boldsymbol{\pi}}^{i}$.
Note that $\boldsymbol{P}_{\boldsymbol{\pi}}\left(\mathcal{R},\mathcal{R}\right)$
is a stochastic matrix since each block matrix is a stochastic matrix
\cite{Durrett1999}. Similarly, the rewards $\boldsymbol{\boldsymbol{R}}_{\boldsymbol{\pi}}$
can also be partitioned into $\boldsymbol{R}_{\boldsymbol{\pi}}^{tr}$
and $\boldsymbol{R}_{\boldsymbol{\pi}}^{rec}$.

The following proof is based on contradiction. Suppose there exists
a policy $\boldsymbol{\pi}^{*}$ that optimizes the expected return,
but does not satisfy the accepting condition of $\mathcal{\mathcal{\mathcal{\mathcal{P}}}}$ with non-zero probability.
Based on Lemma \ref{lemma:accepting set}, the following is true:
$F_{k}^{\mathcal{\mathcal{\mathcal{\mathcal{\mathcal{P}}}}}}\subseteq\ensuremath{\mathcal{T}_{\boldsymbol{\pi}^{*}}},\forall k\in\left\{ 1,\ldots f\right\} $,
where $\ensuremath{\mathcal{T}_{\boldsymbol{\pi}^{*}}}$ denotes the
transient class of Markov chain induced by $\pi^{*}$ on $\mathcal{\mathcal{\mathcal{\mathcal{\mathcal{P}}}}}$.
First, consider a state $x_{R}\in\mathcal{R}_{\boldsymbol{\pi}^{*}}^{j}$
and let $\boldsymbol{P}_{\boldsymbol{\pi}^{*}}^{x_{R}R_{j}}$ denote
a row vector of $\boldsymbol{P}_{\boldsymbol{\pi}^{*}}^{n}\left(\mathcal{R},\mathcal{R}\right)$
that contains the transition probabilities from $x_{R}$ to the states
in the same recurrent class $\mathcal{R}_{\pi^{*}}^{j}$ after $n$
steps. The expected return of $x_{R}$ under $\boldsymbol{\pi}^{*}$
is then obtained from (\ref{eq: utility_function}) as 
\[
U_{\boldsymbol{\pi}^{*}}^{rec}\left(x_{R}\right)=\stackrel[n=0]{\infty}{\sum}\gamma^{n}\left[\boldsymbol{0}_{k_{1}}^{T}\:\boldsymbol{P}_{\pi^{*}}^{x_{R}R_{j}}\:\boldsymbol{0}_{k_{2}}^{T}\right]\boldsymbol{R}_{\boldsymbol{\pi}^{*}}^{rec},
\]
where $k_{1}=\sum_{i=1}^{j-1}N_{i}$, $k_{2}=\sum_{i=j+1}^{n}N_{i}$.
Since $\ensuremath{\mathcal{R}_{\pi^{*}}^{j}}\cap F_{i}^{\mathcal{P}}=\emptyset,\forall i\in\left\{ 1,\ldots f\right\} $,
by the designed reward function, all entries of $\boldsymbol{R}_{\boldsymbol{\pi}^{*}}^{rec}$
are zero. We can conclude $U_{\boldsymbol{\pi}^{*}}^{rec}\left(x_{R}\right)=0$.
To show contradiction, the following analysis will show that $U_{\bar{\boldsymbol{\pi}}}^{rec}\left(x_{R}\right)>U_{\boldsymbol{\pi}^{*}}^{rec}\left(x_{R}\right)$
for any policy $\bar{\boldsymbol{\pi}}$ that satisfies the accepting
condition of $\mathcal{P}$. Thus, it's true that there
exists $\mathcal{R}_{\bar{\boldsymbol{\pi}}}^{j}$ such that $\ensuremath{\mathcal{R}_{\bar{\boldsymbol{\pi}}}^{j}}\cap F_{k}^{\mathcal{\mathcal{\mathcal{\mathcal{P}}}}}\neq\emptyset,\forall k\left\{ 1,\ldots f\right\} $.
We use $\underline{\gamma}$ and $\overline{\gamma}$ to denote the
lower and upper bound of $\gamma$.

\textbf{Case 1:} If $x_{R}\in\mathcal{R}_{\bar{\boldsymbol{\pi}}}^{j}$,
there exist states such that $x_{\varLambda}\in\mathcal{R}_{\bar{\boldsymbol{\pi}}}^{j}\cap F_{i}^{\mathcal{\mathcal{\mathcal{\mathcal{P}}}}}$.
From Lemma \ref{lemma:accepting set}, the entries in $\boldsymbol{R}_{\bar{\boldsymbol{\pi}}}^{rec}$
corresponding to the recurrent states in $\mathcal{R}_{\bar{\boldsymbol{\pi}}}^{j}$
have non-negative rewards and at least there exist $f$ states in
$\mathcal{R}_{\bar{\boldsymbol{\pi}}}^{j}$ from different accepting
sets $F_{i}^{\mathcal{R}}$ with positive reward $1-r_{F}$. From (\ref{eq: utility_function}),
$U_{\bar{\boldsymbol{\pi}}}^{rec}\left(x_{R}\right)$ can be lower
bounded as 
\[
\begin{aligned}U_{\bar{\boldsymbol{\pi}}}^{rec}\left(x_{R}\right) & \geq\stackrel[n=0]{\infty}{\sum}\underline{\gamma}^{n}\left(P_{\bar{\boldsymbol{\pi}}}^{x_{R}x_{\varLambda}}r_{F}\right)\end{aligned}
>0,
\]
where $P_{\bar{\boldsymbol{\pi}}}^{x_{R}x_{\varLambda}}$ is the transition
probability from $x_{R}$ to $x_{\varLambda}$ in $n$ steps. We can
conclude in this case $U_{\bar{\boldsymbol{\pi}}}^{rec}\left(x_{R}\right)>U_{\boldsymbol{\pi}^{*}}^{rec}\left(x_{R}\right)$.

\textbf{Case 2:} If $x_{R}\in\mathcal{T}_{\bar{\boldsymbol{\pi}}}$,
there are no states of any accepting set $F_{i}^{\mathcal{P}}$ in
$\mathcal{T}_{\bar{\boldsymbol{\pi}}}$. As demonstrated in \cite{Durrett1999},
for a transient state $x_{tr}\in\mathcal{T}_{\bar{\boldsymbol{\pi}}}$,
there always exists an upper bound $\Delta<\infty$ such that $\stackrel[n=0]{\infty}{\sum}p^{n}\left(x_{tr},x_{tr}\right)<\Delta$,
where $p^{n}\left(x_{tr},x_{tr}\right)$ denotes the probability of
returning from a transient state $x_{T}$ to itself in $n$ time steps.
In addition, for a recurrent state $x_{rec}$ of $\mathcal{R}_{\bar{\boldsymbol{\pi}}}^{j}$,
it is always true that 
\begin{equation}
\stackrel[n=0]{\infty}{\sum}\gamma^{n}p^{n}\left(x_{rec},x_{rec}\right)>\frac{1}{1-\gamma^{\overline{n}}}\bar{p},\label{eq:case2_1}
\end{equation}
where there exists $\overline{n}$ such that $p^{\overline{n}}\left(x_{rec},x_{rec}\right)$
is nonzero and can be lower bounded by $\bar{p}$ \cite{Durrett1999}.
From (\ref{eq: utility_function}), one has 
\begin{equation}
\begin{aligned}\boldsymbol{U}_{\bar{\boldsymbol{\pi}}}^{tr} & >\stackrel[n=0]{\infty}{\sum}\left(\stackrel[j=0]{n-1}{\prod}\boldsymbol{\boldsymbol{\gamma}}_{\bar{\boldsymbol{\pi}}}^{tr}\right)\ldotp\boldsymbol{P}_{\bar{\boldsymbol{\pi}}}^{tr}\boldsymbol{P}_{\bar{\boldsymbol{\pi}}}^{n}\left(\mathcal{R},\mathcal{R}\right)\boldsymbol{R}_{\pi}^{rec}\\
 & {\color{black}>\underline{\gamma}^{n}\ldotp\boldsymbol{P}_{\bar{\boldsymbol{\pi}}}^{tr}\boldsymbol{P}_{\bar{\boldsymbol{\pi}}}^{n}\left(\mathcal{R},\mathcal{R}\right)\boldsymbol{R}_{\boldsymbol{\pi}}^{rec}}.
\end{aligned}
\label{eq:case2_2}
\end{equation}
Let $\max\left(\cdot\right)$ and $\min\left(\cdot\right)$ represent
the maximum and minimum entry of an input vector, respectively. The
upper bound $\bar{m}=\left\{ \max\left(\overline{M}\right)\left|\overline{M}<\boldsymbol{P}_{\bar{\pi}}^{tr}\boldsymbol{\bar{P}}\boldsymbol{R}_{\pi}^{rec}\right.\right\} $
and $\bar{m}\geq0$, where $\boldsymbol{\bar{P}}$ is a block matrix
whose nonzero entries are derived similarly to $\bar{p}$ in (\ref{eq:case2_1}).
The utility $U_{\bar{\boldsymbol{\pi}}}^{tr}\left(x_{R}\right)$ can
be lower bounded from (\ref{eq:case2_1}) and (\ref{eq:case2_2})
as $U_{\bar{\boldsymbol{\pi}}}^{tr}\left(x_{R}\right)>\frac{1}{1-\underline{\gamma}^{n}}\bar{m}.$
Since $U_{\boldsymbol{\pi}^{*}}^{rec}\left(x_{R}\right)=0$, the contradiction
$U_{\bar{\boldsymbol{\pi}}}^{tr}\left(x_{R}\right)>0$ is achieved
if $\frac{1}{1-\underline{\gamma}^{n}}\bar{m}$. Thus, there exist
$0<\underline{\gamma}<1$ such that $\gamma_{F}>\underline{\gamma}$
and $r_{F}>\underline{\gamma}$, which implies $U_{\bar{\boldsymbol{\pi}}}^{tr}\left(x_{R}\right)>\frac{1}{1-\underline{\gamma}^{n}}\bar{m}\geq0$.
The procedure shows the contradiction of the assumption that $\pi^{*}$
that does not satisfy the acceptance condition of $\mathcal{\mathcal{\mathcal{\mathcal{\mathcal{P}}}}}$ with non-zero probability
is optimal, and Theorem \ref{thm2} is proved. 

\bibliographystyle{IEEEtran}
\bibliography{references}

\end{document}